\documentclass[twoside,11pt]{article}

 \usepackage[abbrvbib, preprint]{jmlr2e}

\usepackage{hyperref}
\usepackage{times,subfigure}
\usepackage{epsfig,url}
\usepackage{graphicx}
\usepackage{amsmath}
\usepackage{amssymb,bm}
\usepackage{amsfonts}
\usepackage{booktabs}
\usepackage{threeparttable}
\usepackage{appendix}
\usepackage[ruled,linesnumbered]{algorithm2e}

\usepackage{makeidx}  
\usepackage{latexsym}
\usepackage{caption}
\usepackage{marginnote}
\usepackage{indentfirst}

\usepackage{color}
\usepackage{lineno}
\usepackage{multirow}
\usepackage{makecell}
\usepackage{microtype}
\usepackage{url}            
\usepackage{nicefrac}       
\usepackage{mathrsfs}
\usepackage{pgflibraryarrows}
\usepackage{pgflibrarysnakes}
\usepackage{tikz}
\usepackage{pgfplots}

\DeclareMathOperator*{\argmax}{argmax}
\DeclareMathOperator*{\argmin}{argmin}

\DeclareMathOperator*{\diag}{diag}

\DeclareMathOperator*{\F}{F}

\jmlrheading{21}{2020}{1-39}{10/19; Revised
	7/20}{10/20}{19-900}{Fanghui Liu, Xiaolin Huang, Chen Gong, Jie Yang and Li Li}

\ShortHeadings{Learning Data-adaptive Nonparametric Kernels}{Liu, Huang, Gong, Yang and Li}
\firstpageno{1}

\begin{document}

\title{Learning Data-adaptive Non-parametric Kernels}

\author{\name Fanghui Liu~\thanks{The bulk of this work was performed while Fanghui was a PhD student at Shanghai Jiao Tong University.}
	\email fanghui.liu@kuleuven.be \\
	\addr Department of Electrical Engineering, ESAT-STADIUS, KU Leuven, B-3001, Belgium
	\AND
	\name Xiaolin Huang \email xiaolinhuang@sjtu.edu.cn\\
	\addr Institute of Image Processing and Pattern Recognition\\
	 Institute of Medical Robotics, Shanghai Jiao Tong University, Shanghai, 200240, China
	\AND
	\name Chen Gong \email chen.gong@njust.edu.cn \\
	\addr PCA Lab, Key Laboratory of Intelligent Perception and Systems for High-Dimensional Information of Ministry of Education\\
	 School of Computer Science and Engineering, Nanjing University of Science and Technology, 210094, China \\
	Department of Computing, Hong Kong Polytechnic University, Hong Kong SAR, China.
	\AND
	\name Jie Yang~
	\email jieyang@sjtu.edu.cn \\
	\addr Institute of Image Processing and Pattern Recognition\\
	 Institute of Medical Robotics, Shanghai Jiao Tong University, Shanghai, 200240, China
	\AND
	\name Li Li \email li-li@tsinghua.edu.cn \\
	\addr Department of Automation, BNRist, Tsinghua University, 100084, China
}

\editor{John Shawe-Taylor}

\maketitle

\begin{abstract}
In this paper, we propose a data-adaptive non-parametric kernel learning framework in margin based kernel methods.
In model formulation, given an initial kernel matrix, a data-adaptive matrix with two constraints is imposed in an entry-wise scheme.
Learning this data-adaptive matrix in a formulation-free strategy enlarges the margin between classes and thus improves the model flexibility.
The introduced two constraints are imposed either exactly (on small data sets) or approximately (on large data sets) in our model, which provides a controllable trade-off between model flexibility and complexity with theoretical demonstration.
In algorithm optimization, the objective function of our learning framework is proven to be gradient-Lipschitz continuous.
Thereby, kernel and classifier/regressor learning can be efficiently optimized in a unified framework via Nesterov's acceleration.
For the scalability issue, we study a decomposition-based approach to our model in the large sample case.
The effectiveness of this approximation is illustrated by both empirical studies and theoretical guarantees.
Experimental results on various classification and regression benchmark data sets demonstrate that our non-parametric kernel learning framework achieves good performance when compared with other representative kernel learning based algorithms.
\end{abstract}

\begin{keywords}
support vector machines, non-parametric kernel learning, gradient-Lipschitz continuous
\end{keywords}

\section{Introduction}
\label{sec:intro}
Kernel methods \citep{shawe2000support,Sch2003Learning,Steinwart2008SVM} have proven to be powerful in a variety of machine learning tasks, e.g., Support Vector Machines (SVM) for classification \citep{Vapnik2000The,suykens2002least}, Support Vector Regression (SVR) for regression \citep{drucker1997support}, and kernel mean embedding for casual inference \citep{mitrovic2018causal}.
They employ a so-called \emph{kernel} function $k: \mathbb{R}^d \times \mathbb{R}^d \rightarrow \mathbb{R}$ to compute the similarity between any two samples $\bm x_i$, $\bm x_j \in \mathbb{R}^d$ such that $ k(\bm x_i, \bm x_j) = \langle \phi(\bm x_i), \phi(\bm x_j) \rangle_{\mathcal{H}} $, where $\phi: \mathcal{X} \rightarrow \mathcal{H}$ is a non-linear feature map transforming elements of input spaces $\mathcal{X}$ into a reproducing kernel Hilbert space (RKHS) $\mathcal{H}$.
Specifically, for a given kernel, the ``kernel trick" allows for optimization in the kernel-associated hypothesis space without explicit representation of such mapping.

Generally, the performance of kernel methods largely depends on the choice of the kernel.
Traditional kernel methods often adopt a classical kernel, e.g., Gaussian kernel or sigmoid kernel for characterizing the relationship among data points.
Empirical studies suggest that these traditional kernels are not sufficiently flexible to depict the domain-specific characteristics of data affinities or relationships.
To address such limitation, several routes have been explored.
One way is to design sophisticated kernels on specific tasks such as applying optimal assignment kernels \citep{Kriege2016On} to graph classification, developing a kernel based on triplet comparisons \citep{kleindessner2017kernel}, designing the class of tessellated kernels \citep{colbert2020convex}, or even breaking the restriction of positive definiteness on kernels \citep{Cheng2004Learning,Schleif2015Indefinite,Ga2016Learning}.
Apart from these well-designed kernels, a series of research studies aim to automatically learn effective and flexible kernels from data, known as \emph{learning kernels}.
Algorithms for learning kernels can be roughly grouped into two categories: parametric kernel learning and non-parametric kernel learning.

\subsection{Review of Kernel Learning}
\label{sec:review}

In parametric kernel learning, the (learned) kernel function $k(\cdot,\cdot)$ or the kernel matrix $\bm K = [k(\bm x_i, \bm x_j)]_{n \times n}$ on the training data is assumed to admit a specific parametric form, and then the relevant parameters are learned according to the given data.
The earliest work is proposed by \cite{Lanckriet2004Learning}, in which they consider training SVM along with optimizing a linear combination of several pre-given positive semi-definite (PSD) matrices $\{ \underline{\bm K}_t \}_{t=1}^s$ subject to a bounded trace constraint, i.e., $\bm K = \sum_{t=1}^{s}\mu_t \underline{\bm K}_t$ with $\mathrm{tr}(\bm K) \leq c$.
Specifically, to ensure the learned kernel matrix to be PSD, one can directly use the constraint $\bm K \in \mathcal{S}^n_+$ (the cone of $n \times n$ positive semi-definite matrices); or consider a nonnegaitve linear combination of $\{ \underline{\bm K}_t \}_{t=1}^s$, i.e., $\mu_t \geq 0$.
Based on the above two schemes, the kernel (matrix) learning is transformed to learn the combination weights. 
Accordingly, the parameters in SVM and the associated weights can be learned by solving a semi-definite programming optimization problem in a unified framework.

The above parametric kernel learning framework spawns the new field of \emph{multiple kernel learning} (MKL) \citep{bach2004multiple,Varma2009More,liu2019absent}.
It aims to learn a good combination of some predefined kernels (or kernel matrices) for rich representations.
For example, the weight vector $\bm \mu =[\mu_1, \mu_2, \dots, \mu_s]^{\!\top}$ can be restricted by the conic sum (i.e., $\mu_t \geq 0$), the convex sum (i.e., $\mu_t \geq 0$ and $\sum_{t=1}^{s}\mu_t = 1$), or various regularizers such as $\ell_1$ norm, mixed norm, and entropy-based formulations, see a survey \citep{G2011Multiple}.
By doing so, MKL would generate a ``broader" kernel to enhance the representation ability for data.
Based on the idea of MKL, there are several representative approaches to learn effective kernels by exploring the data information, including:
i) hierarchical kernel learning (HKL) \citep{bach2008exploring,jawanpuria2015generalized} learns from a set of base kernels assumed to be embedded on a directed acyclic graph;
ii) spectral mixture models \citep{argyriou2005learning,Wilson2013Gaussian,jean2018semi} aim to learn the spectral density of a kernel in a parametric scheme for discovering flexible statistical representations in data;
iii) kernel target alignment \citep{cristianini2002kernel,Cortes2012Algorithms} seeks for the ``best" kernel matrix by maximizing the similarity between $\bm K$ and the ideal kernel $\bm y \bm y^{\!\top}$ with the label vector $\bm y$.
Here the used ideal kernel can directly recognize the training data with 100\% accuracy, and thus can be used to guide the kernel learning task.
Current works on this direction often assume that the learned kernel matrix  $\bm K$ is in a parametric way, e.g., an MKL or mixtures of spectral density form, see \citep{Lanckriet2004Learning,AmanNIPS2016,li2019implicit} and references therein.

Instead of assuming specific forms for the learned kernel in parametric kernel learning, nonparametric kernel learning is another way to acquire a positive definite kernel (matrix) in a data-specific manner.
Typical examples include:
\cite{Lanckriet2004Learning} directly consider $\bm K \in \mathcal{S}_+^n$ without extra parametric forms in their optimization problem, which results in a nonparametric kernel learning framework.
Such nonparametric kernel learning model is further explored by learning a low-rank kernel matrix \citep{kulis2009low}, imposing the pairwise constraints with side/prior information \citep{Hoi2007Learning}, or local geometry information \citep{Lu2009Geometry}.
These non-parametric kernel learning based problems are usually solved by the standard semi-definite programming or an efficient saddle-point optimization algorithm \citep{Zhuang2011Anonk}.
Besides, \cite{Jain2012Metric} investigate the equivalence between non-parametric kernel learning and Mahalanobis metric learning, and accordingly propose a non-parametric model seeking for a PSD matrix $\bm W$ in a learned kernel $\phi(\bm x)^{\!\top} \bm W  \phi(\bm x')$ via the LogDet divergence.

\subsection{Contributions}


In this paper, we propose a \textbf{D}ata-\textbf{A}daptive \textbf{N}onparametric \textbf{K}ernel (DANK) learning framework that can be seamlessly embedded to support vector machines (SVM) and support
vector regression (SVR).
A low-rank data-adaptive matrix in a suitable solving space is learned to enlarge the margin between classes and effectively control the model complexity.
Further, by virtue of the gradient-Lipschitz continuous property of the considered objective function, the learning task can be efficiently solved by a projected gradient method with Nesterov’s acceleration.
Different from previous non-parametric kernel learning models optimized by semi-definite programming, the employed optimization algorithm in our DANK model is efficient in large scale cases.

To be specific, in our DANK model, a low-rank matrix $\bm F \in \mathbb{R}^{n \times n}$ in a bounded feasible region is imposed on a pre-given kernel matrix $\bm K$ in a point-wise strategy, i.e., $\bm F \odot \bm K$, where $\odot$ denotes the Hadamard product between two matrices.
This \emph{formulation-free} strategy contributes to adequate model flexibility as a result.
The used low-rank constraint and the bounded constraint restrict the degree of freedom of $\bm F$, of which the design scheme is independent of the pre-given kernel matrix.
	Thereby we can restrict the flexibility of $\bm F$ so as to control the model complexity in a clear way.
Here we take a synthetic classification data set \emph{clowns} to illustrate this controllable trade-off between the model complexity and flexibility.
The initial kernel matrix $\bm K$ is given by the classical Gaussian kernel $k(\bm x_i, \bm x_j) = \exp(-{\| \bm x_i - \bm x_j \|_2^2}/{2\sigma^2})$ with the width $\sigma$.
Figure~\ref{flexi} shows that, in the left panel, the baseline SVM with an inappropriate $\sigma=1$ lacks model flexibility and cannot precisely adapt to the data.
However, based on the same $\bm K$, by optimizing the adaptive matrix $\bm F$, our DANK model shows good flexibility to fit the complex data distribution, leading to a desirable decision surface.
Comparably, in the right panel, even under the condition that $\sigma$ is well tuned, the baseline SVM fails to capture the local property of the data (see the brown ellipses).
Instead, by learning $\bm F$, our DANK model still yields a more accurate boundary than SVM to fit the data points.
Specifically, the model complexity is controlled by the introduced low-rank and bounded constraints.
We find that, the generated classification boundary in Figure~\ref{flexi} prohibits locally linearly separable, which avoids over-fitting by an extremely flexible decision surface \citep{torr2011locally}.
\begin{figure}
	\centering
	\subfigure[$\sigma=1$ (pre-given)]{\label{sigma01}
		\includegraphics[width=0.352\textwidth]{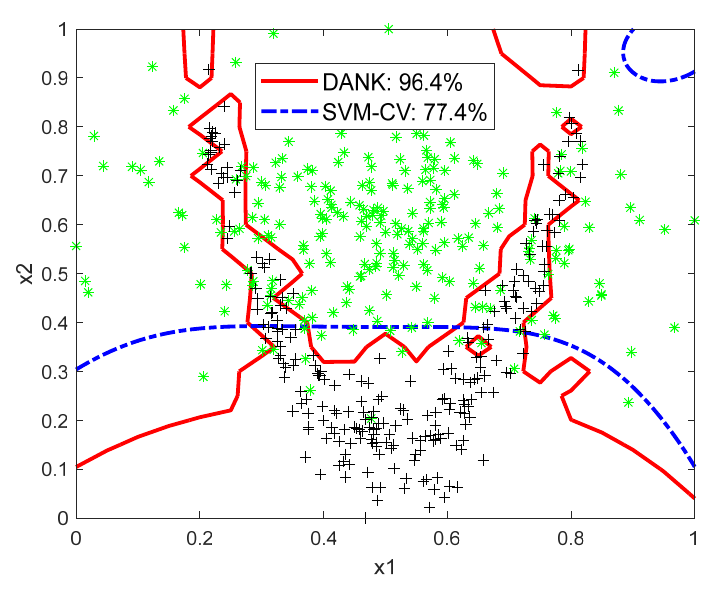}}
	\hspace{0.5cm}
	\subfigure[$\sigma=0.07$ (cross validation)]{\label{sigma1}
		\includegraphics[width=0.352\textwidth]{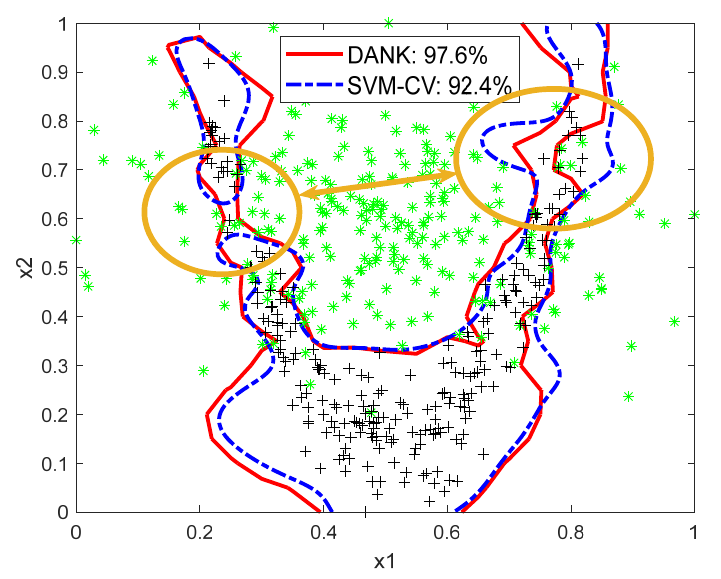}}
	\caption{Classification boundaries of SVM (in blue dash line) and our DANK model (in red solid line) with the Gaussian kernel on the \emph{clowns} data set under different kernel width values $\sigma$.}\label{flexi}
\end{figure}

The main contributions of this paper lie in the following three folds:
\begin{itemize}
	\item In model formulation, we propose a data-adaptive nonparametric kernel learning framework termed ``DANK" to enhance the model flexibility and data adaptivity, which is then seamlessly embedded to margin based kernel methods (SVM and SVR) for classification and regression tasks. 
	Specifically, the introduced constraints are demonstrated to be effective on a controllable trade-off between the model flexibility and complexity;
	\item In algorithm optimization, the DANK model and the induced classification/regression model can be formulated as a max-min optimization problem in a unified framework.
	The related objective function is proven to be gradient-Lipschitz continuous, and thus can be directly solved by a projected gradient method with Nesterov's acceleration;
	\item In scalability issue, we propose a decomposition based approach to our model by omitting the non-separable low-rank constraint in large sample case. 
	The effectiveness of our decomposition-based scalable approach is demonstrated by both theoretical and empirical studies.
\end{itemize}
Besides, the experimental results on several classification and regression benchmark data sets demonstrate the effectiveness of the proposed DANK framework over other representative kernel learning based methods.

This paper shares the basic ideas with our previous conference work \citep{Liu2018AAAIme} but is totally different in model formulation, algorithm optimization, and scaling in large sample cases.
First, in model formulation, we impose an additional low-rank constraint and a bounded constraint on $\bm F$, which effectively control the model flexibility with theoretical demonstration. 
Apart from SVM to which our DANK model is embedded for classification, the proposed DANK model is also extended to SVR for regression tasks.
Second, we develop a Nesterov's smooth method to solve the designed optimization problem, which requires the discussion on its gradient-Lipschitz continuous property.
Third, we conduct a decomposition-based scalable approach on DANK with theoretical guarantees and experimental validation for large scale situations.
Lastly, we provide more experimental results on popular benchmarks.

\subsection{Notation}
\label{sec:notation}
We start with notations this paper.

{\bf Matrices, vectors and elements:} We take $\bm A$, $\bm a$ to be a matrix and a vector, of which the entries are $A_{ij}$ and $a_i$, respectively. Denote $\bm I_n$ as the $n \times n$ identity matrix, $\bm 0$ as a zero matrix or vector with the appropriate size, and $\bm 1_n$ as the $n$-dimensional vector of all ones.

{\bf Sets:} The set $\{1,2,\cdots,n\}$ is written as $[n]$. We call $\{\mathcal{V}_1, \mathcal{V}_2, \cdots, \mathcal{V}_s\}$ an $s$-partition of $[n]$ if $\mathcal{V}_1 \cup \cdots \cup \mathcal{V}_s = [n]$ and $\mathcal{V}_p \cap \mathcal{V}_q = \emptyset $ for $p \neq q$. Let $|\mathcal{V}|$ denote the cardinality of the set $\mathcal{V}$.
We take the notation $\mathcal{S}^n$ as the set of $n \times n$ symmetric matrices and $\mathcal{S}^n_+$ as the $n \times n$ PSD cone.

{\bf Singular value decomposition (SVD):} Given a matrix $\bm A \in \mathbb{R}^{n \times d}$ and its rank $r = \text{rank}(\bm A)$, a (compact) singular value decomposition (SVD) is defined as
$\bm{A} = {\bm U}{\bm \Sigma}{\bm V}^{\!\top} = \sum_{i=1}^{r} \sigma_{i}(\bm A) \bm u_i \bm v^{\!\top}_i$, 
where $\bm U$, $\bm \Sigma$, $\bm V$ are an $n \times r$ column-orthogonal matrix, an $r \times  r$ diagonal matrix with its diagonal element $\sigma_{i}(\bm A)$, and a $d \times r$ column-orthogonal matrix, respectively. If $\bm A$ is PSD, then $\bm U = \bm V$. 
Accordingly, the singular value soft-thresholding operator is defined as 
$\mathcal{J}_{\tau}(\bm A)=\bm U_{\!\bm A} \mathcal{S}_{\tau}(\bm \Sigma_{\bm A}){\bm V}^{\!\top}_{\bm A}$ with the SVD: $\bm{A=U}{\bm \Sigma}{\bm V}^{\top}$ and the soft-thresholding operator is $\mathcal{S}_{\tau}(A_{ij})=\text{sign}(A_{ij})\max(0, |A_{ij} |-\tau)$.

{\bf Matrix norms:} We use four matrix norms in this paper

{\centering Frobenius norm: $\| \bm A \|_{\F} = \sqrt{\sum_{i,j} A_{ij}^2 } = \sqrt{\sum_{i}\sigma_i^2(\bm A)}$}\,.\\
Spectral norm: $\| \bm A \|_2 = \max \limits_{\| \bm x \|_2 = 1} \| \bm A \bm x \|_2 = \sigma_{\max}(\bm A) \leq \| \bm A \|_{\F}$. \\
Nuclear norm: $\| \bm A \|_* = \sum_{i} \sigma_i(\bm A)$.\\
Any square matrix satisfies $\mathrm{tr}(\bm A ) \leq \| \bm A \|_*$. If $\bm A$ is PSD, then we have $\mathrm{tr}(\bm A ) = \| \bm A \|_*$ and $\| \bm A \|_2 = \lambda_{\max}(\bm A)$, where $\lambda_{\max}(\bm A)$ denotes the largest eigenvalue of $\bm A$.

\subsection{Paper Organization}
The paper is organized as follows. In Section \ref{sec:knplsvm}, we introduce the proposed DANK model embedded in SVM, mainly on the model formulation in Section~\ref{sec:modelfor}, benefits of the introduced constraints in Section~\ref{sec:benefits}, and out-of-sample extensions in Section~\ref{sec:ose}.
The model optimization is presented in Section \ref{sec:knplopt}: Section~\ref{sec:gradl} studies the gradient-Lipschitz continuous property and Section~\ref{sec:nest} applies Nesterov's smooth optimization method to solve our model.
Scalability of our nonparametric kernel model is addressed in Section \ref{sec:knplapp}.
Besides, in Section \ref{sec:knplsvr}, we extend our DANK model to SVR for regression.
The experimental results on popular benchmark data sets are presented in Section \ref{sec:experiment}.
Section \ref{sec:conclusion} concludes the entire paper.
Proofs are provided in the Appendices.

\section{The DANK Model in SVM}
\label{sec:knplsvm}
In this section, we incorporate the proposed DANK model into SVM for classification, discuss benefits of the introduced constraints, and extend it to test data for out-of-sample extension.
Denote $\mathcal{X} \subseteq \mathbb{R}^d$ as a compact metric space, and $\mathcal{Y}=\{ -1, 1\}$ as the label space, we assume that a sample set $\mathcal{Z} = \{  (\bm x_i, y_i) \}_{i=1}^n $ is drawn from a non-degenerate Borel probability measure $\rho$ on $\mathcal{X} \times \mathcal{Y}$.
We focus on binary classification problems for the ease of description and it can be extended to multi-classification tasks.


\subsection{Model Formulation}
\label{sec:modelfor}
We begin with the SVM formulation and then introduce our DANK model in SVM.
The hard-margin SVM aims to learn a linear classifier $f(\bm x; \bm w, b) = \mbox{sign}(\bm w^{\!\top}\bm x + b) \in \{ -1, +1\}$ with $\bm w$ and $b$ that determine the decision hyperplane.
This is conducted by maximizing the distance between the nearest training samples of the two classes (a.k.a the \emph{margin} $\gamma = 1/\| \bm w \|_2$), as this way reduces the model generalization error \citep{Vapnik2000The}.
While the data are not linearly separable in most practical settings, the hard-margin SVM is subsequently extended to a soft-margin SVM with an implicit mapping $\phi(\cdot)$ for a non-linear decision hyperplane.
Mathematically, the soft-margin SVM aims to maximize the margin $\gamma$ (or minimize $\| \bm w \|_2^2$) and minimize the slack penalty $\sum_{i=1}^{n} \xi_i$ with the following formulation
\begin{equation}\label{primalsvm}
\begin{split}
&\min \limits_{{\bm w},b, \bm \xi} \frac{1}{2}\| {\bm w} \|_2^2 + C \sum_{i=1}^{n} \xi_i \\
&\mbox{s.t.} ~~ y_i (\bm w^{\!\top} \phi(\bm x_i)+b) \geq 1 - \xi_i,~ \xi_i \geq 0,~ i=1,2,\cdots,n \,,
\end{split}
\end{equation}
where $\bm \xi = [\xi_1, \xi_2, \cdots, \xi_n]^{\!\top}$ is the slack variable and $C$ is the balance parameter.
As illustrated by \cite{Vapnik2000The}, the dual form of problem \eqref{primalsvm} is given by
\begin{equation}\label{dualSVM}
\begin{split}
&\mathop{\mathrm{max}}\limits_{{\bm \alpha} \in \mathcal{A}}~~ \bm{1}^{\!\top}{\bm \alpha}  - \frac{1}{2}{\bm \alpha}^{\!\top}\bm{YKY}{\bm \alpha} \,,
\end{split}
\end{equation}
where $\bm Y=\diag(\bm y)$ is the label matrix, $\bm K = [k(\bm x_i, \bm x_j)]_{n \times n}$ is the (pre-given) Gram matrix  satisfying $k(\bm x_i, \bm x_j) = \langle \phi(\bm x_i), \phi(\bm x_j) \rangle_{\mathcal{H}}$, and the constraint set is given by $\mathcal{A}=\{ {\bm \alpha}\in \mathbb{R}^n: {\bm \alpha}^{\top}\bm{y}=0,~\bm 0\leq {\bm \alpha} \leq C\bm{1} \}$. Without loss of generality, we assume that the kernel function is bounded, i.e., $\kappa := \sup_{\bm x, \bm x' \in \mathcal{X}}|k(\bm x, \bm x')| < \infty $.

Theoretical results on learning kernels, mainly on multiple kernel learning \citep{srebro2006learning,hussain2011improved} demonstrate that, 
	to achieve a tight bound of the estimation error (namely the gap between empirical error and expected error), a learned SVM classifier is better to admit a large margin $\gamma$, and to effectively control the complexity of the hypothesis space.
	Despite that the above theoretical results on multiple kernel learning cannot be directly applied to non-parametric kernel learning models as the learned kernel is sample-dependent and implicit, their results are able to motivate us to design our non-parametric model. On one hand, an adaptive matrix $\bm F$ is introduced into problem~\eqref{dualSVM} to increase the margin $\gamma$; on the other hand, two constraints are considered to control the model complexity. Mathematically, our DANK model is
\begin{equation}\label{main}
\begin{split}
&\min \limits_{\bm{F}\in \mathcal{S}^n_+}  \max\limits_{\bm \alpha \in \mathcal{A}} ~\bm{1}^{\!\top}\!{\bm \alpha} -\frac{1}{2} {\bm \alpha}^{\!\top}\bm{Y}\big(\bm{F}\!\odot\!\bm{K}\big)\bm{Y}{\bm \alpha} \\
&\mbox{s.t.} ~~ \| \bm{F} - \bm{1}\bm{1}^{\!\top} \! \|_{\text{F}}^2 \leq R^2,~\mbox{rank}(\bm F) < r\,,
\end{split}
\end{equation}
where $R$ refers to the bounded region size, $\mbox{rank}(\bm F) $ denotes the rank of $\bm F$, and $r \leq n$ is a given integer.
The constraint $\bm F \in \mathcal{S}^n_+$ is given to ensure that the learned kernel matrix $\bm{F}\odot\bm{K}$ is still a PSD one.\footnote{It is admitted by Schur Product Theorem \citep{Styan1973Hadamard} which relates positive semi-definite matrices to the Hadamard product.} 
Since we do not specify the parametric form of $\bm F$, we can obtain a nonparametric kernel matrix $\bm F \odot \bm K$ and thus our DNAK model is nonparametric.
Due to the non-convexity of the used rank constraint in problem~\eqref{main}, we consider the \emph{nuclear norm} $\| \cdot \|_*$ instead, which is the best convex lower bound of the non-convex rank function \citep{recht2010guaranteed} and can be minimized efficiently.
Accordingly, we relax the constrained optimization problem in Eq.~\eqref{main} to a unconstrained problem by absorbing the two original constraints to the objective function.
Moreover, following the min-max approach \citep{boyd2004convex}, problem~\eqref{main} can be reformulated as
\begin{equation}\label{mainrank}
\begin{split}
&\max\limits_{\bm \alpha \in \mathcal{A}}\min\limits_{\bm{F}\in \mathcal{S}^n_+}  \bm{1}^{\!\top}{\bm \alpha} -\frac{1}{2} {\bm \alpha}^{\!\top}\bm{Y}\big(\bm{F}\odot \bm{K}\big)\bm{Y} {\bm \alpha} + {\eta} \| \bm{F} - \bm{1}\bm{1}^{\!\top}  \|_{\F}^2 +\tau \eta \| \bm{F} \|_* \,,
\end{split}
\end{equation}
where $\eta$, $\tau$ are two regularization parameters. Here we denote $\| \bm{F} - \bm{1}\bm{1}^{\!\top}  \|_{\F}^2$ as the centering regularizer.
In problem~\eqref{mainrank}, the inner minimization problem with respect to $\bm{F}$ is a convex conic programming, and the outer maximization problem is a point-wise minimum of concave quadratic functions of ${\bm \alpha}$.
As a consequence, problem~\eqref{mainrank} is convex, and strong duality holds by Slater's condition \citep{boyd2004convex}.
The optimal values of the primal and dual form of kernel learning based SVM problems will be equal.
Accordingly, when we learn the kernel matrix in problem~\eqref{mainrank}, the objective function value of its primal problem would decrease, which in turn enlarges the margin $\gamma$ for increasing model flexibility.
The introduced two regularizers in Eq.~\eqref{mainrank} are beneficial to control the model complexity. We briefly explain this here and detail in the next subsection.
Intuitively speaking, when $\bm F$ is chosen as the all-one matrix, the DANK model in Eq.~\eqref{mainrank} degenerates to a standard SVM problem.
The considered low-rank regularizer forces $\bm F$ to be endowed with the low-rank structure enjoyed by the all-one matrix.
The introduced centering regularizer restricts the adaptive matrix $\bm F$ to vary around the all-one matrix in a small range. This scheme is also able to prevent $\bm F$ from dropping to a trivial solution $\bm{F}=\bm{0}_{n\times n}$.
By the above two regularizers, we can effectively restrict the complexity of $\bm F$, and further to control the complexity of the whole model. 

\subsection{Benefits of the Used Constraints}
\label{sec:benefits}
In this subsection, we aim to demonstrate two merits of the introduced constraints/regularizers.
First, although the hard constraints are substituted by two regularizers in Eq.~\eqref{mainrank}, the optimal solution $\bm F^*$ of our unconstrained optimization problem can be still restricted in a bounded set $\mathcal{F}$.
	This property will be beneficial to control the model complexity.
	Second, we elucidate that the learned kernel matrix exhibits a fast eigenvalue decay by the used constraints/regularizers, which would be helpful to achieve good generalization properties.

\subsubsection{The boundedness of the optimal solution}
For notational simplicity, we denote the objective function in Eq.~\eqref{mainrank} as
\begin{equation*}
H({\bm \alpha},\bm{F})=\bm{1}^{\!\top}{\bm \alpha} -\frac{1}{2} {\bm \alpha}^{\!\top}\bm{Y}\big(\bm{F}\odot \bm{K}\big)\bm{Y} {\bm \alpha} + {\eta} \| \bm{F} - \bm{1}\bm{1}^{\!\top}  \|_{\F}^2 +\tau \eta \| \bm{F} \|_*\,,
\end{equation*}
of which the optimal solution $({\bm \alpha}^*,\bm{F}^*)$ is a saddle point of $H({\bm \alpha},\bm{F})$ due to the property of the max-min problem~\eqref{mainrank}.
It is easy to check $H({\bm \alpha},\bm{F}^*) \leq H({\bm \alpha}^*,\bm{F}^*) \leq H({\bm \alpha}^*,\bm{F})$ for any feasible $\bm \alpha$ and $\bm F$.
Further, we define the following function
\begin{equation}\label{falpha}
h(\bm \alpha) := H({\bm \alpha},\bm{F}^*)  = \min \limits_{\bm{F}\in \mathcal{S}^n_+} H({\bm \alpha},\bm{F})\,,
\end{equation}
which is concave since $h(\cdot)$ is the minimum of a sequence of concave functions.
The optimal solution $\bm F^*$ of problem~\eqref{mainrank} can be restricted in a bounded set $\mathcal{F}$ by the following Lemma.
\begin{lemma}\label{fbound}
	Problem \eqref{mainrank} admits the following equivalent formulation such that $\bm F$ can be optimized in a bounded region
	\begin{equation}\label{falphab}
	\max\limits_{\bm \alpha \in \mathcal{A}}\min\limits_{\bm{F}\in \mathcal{S}^n_+} H(\bm \alpha, \bm F) = \max\limits_{\bm \alpha \in \mathcal{A}} \underbrace{\min\limits_{\bm{F}\in \mathcal{F}} H(\bm \alpha, \bm F) }_{\triangleq  h(\bm \alpha) }\,,
	\end{equation}
	where the feasible region on $\bm F$ is defined by $\mathcal{F}:=\Big\{ \bm F \in \mathcal{S}^n_+: \lambda_{\max}(\bm F) \leq n - \frac{\tau}{2}+\frac{nC^2}{4\eta}\lambda_{\max}(\bm K) \Big\}$ as a nonempty subset of $\mathcal{S}^n_+$.
\end{lemma}
\begin{proof}
	The key of the proof is to obtain the optimal solution $\bm F^*$ over $\mathcal{S}^n_+$, i.e., $\bm F^* = \argmin\limits_{\bm{F}\in \mathcal{S}^n_+} H(\bm \alpha, \bm F)$ in problem~\eqref{mainrank}.
	By virtue of the following expression\footnote{We use the formula $\bm{x}^{\!\top}\bm{A}\odot\bm{By}= \mathrm{tr}(\bm{D}_x\bm{A}\bm{D}_y\bm{B}^{\top})$ with $\bm{D}_x=\diag(\bm{x})$ and $\bm{D}_y=\diag(\bm{y})$.}
		\begin{equation*}
		\frac{1}{2} {\bm \alpha}^{\!\top}\bm{Y}\big(\bm{F}\!\odot \!\bm{K}\big)\bm{Y} {\bm \alpha} = \mathrm{tr} \left[ \diag({\bm \alpha}^{\top}\bm{Y})\bm{K}\diag({\bm \alpha}^{\top}\bm{Y}) \bm F \right] \,,
		\end{equation*}
		and denote
		\begin{equation}\label{gammadef}
		\bm \Gamma \equiv \bm \Gamma(\bm \alpha) := \frac{1}{4\eta}\diag({\bm \alpha}^{\top}\bm{Y})\bm{K}\diag({\bm \alpha}^{\top}\bm{Y}) \,,
		\end{equation}
		then finding $\bm F^*$ is equivalent to consider the following problem
		\begin{equation}\label{Smain}
		\bm F^* := \argmin \limits_{\bm{F}\in \mathcal{S}^n_+}~~   -\!2 \mathrm{tr}\Big[ \eta \bm \Gamma(\bm \alpha) \bm{F}\Big] + \eta \| \bm{F} - \bm 1 \bm 1^{\!\top} \|_{\mathrm{F}}^2 + \tau \eta \| \bm F \|_*
		\,,
		\end{equation}
		where we omit the irrelevant term $\bm 1^{\!\top}\bm \alpha$ independent of the optimization variable $\bm F$ in problem~\eqref{mainrank}.
	Further, due to the independence of $\bm \Gamma(\bm \alpha)$ on $\bm F$, problem~\eqref{Smain} can be reformulated as
	\begin{equation}\label{subF}
	\bm F^* = \argmin \limits_{\bm{F}\in \mathcal{S}^n_+} \| \bm{F} -\bm 1 \bm{1}^{\!\top}- \bm \Gamma(\bm \alpha)\|_\text{F}^2 + \tau \| \bm F \|_*\,.
	\end{equation}
	Note that the regularization parameter $\eta$ is implicitly included in $\bm \Gamma (\bm \alpha)$.
	Following \citep{cai2010singular}, we can directly obtain the optimal solution of problem~\eqref{subF} with
	\begin{equation*}
	\bm F^* \equiv \bm F (\bm \alpha)  = \mathcal{J}_{\frac{\tau}{2}}(\bm 1 \bm{1}^{\!\top}+\bm \Gamma(\bm \alpha))\,,
	\end{equation*}
	where we use the singular value thresholding operator $\mathcal{J}_{\frac{\tau}{2}}(\cdot)$ as the proximity operator associated with the nuclear
	norm, refer to Theorem 2.1 in \citep{cai2010singular} for details.
	Based on its closed-form, $\lambda_{\max}(\bm F^*)$ can be upper bounded by
	\begin{equation}\label{eigenf}
	\begin{split}
	\lambda_{\max}(\bm F^*) &=  \lambda_{\max}\Big(\!\mathcal{J}_{\frac{\tau}{2}}\big(\bm 1 \bm{1}^{\!\top}\!+\! \bm \Gamma(\bm \alpha)\big)\!\Big)\! = \! \lambda_{\max}\Big(\bm 1 \bm{1}^{\!\top}\!+\! \bm \Gamma(\bm \alpha) \Big) \!-\! \frac{\tau}{2} \\
	&\leq \lambda_{\max}(\bm 1 \bm{1}^{\!\top}\!) \!+\! \frac{1}{4\eta}\lambda_{\max}\Big(\diag({\bm \alpha}^{\!\top}\bm{Y})\bm{K}\diag({\bm \alpha}^{\top}\bm{Y})\Big) \!-\! \frac{\tau}{2}\\
	&= n + \frac{1}{4\eta} \Big\| \diag({\bm \alpha}^{\!\top}\bm{Y})\bm{K}\diag({\bm \alpha}^{\top}\bm{Y})\Big\|_2 - \frac{\tau}{2}\\
	&\leq n + \frac{1}{4\eta}  \big\| \diag({\bm \alpha}^{\top}\bm{Y}) \big\|^2_2 \big\| \bm K \big\|_2 - \frac{\tau}{2}\\
	&\leq n - \frac{\tau}{2}+\frac{nC^2}{4\eta}\lambda_{\max}(\bm K)\,,
	\end{split}
	\end{equation}
	where the first inequality uses the property of maximum eigenvalues, \emph{i.e.}, $\lambda_{\max}(\bm A + \bm B) \leq \lambda_{\max}(\bm A) + \lambda_{\max}(\bm B)$ for any $\bm A, \bm B \in \mathcal{S}_+^n$, and the last inequality admits by $\| \bm \alpha\|_2^2 \leq nC^2$.

\end{proof}
{\bf Remark:} The centering regularizer ${\eta} \| \bm{F} - \bm{1}\bm{1}^{\!\top}  \|_{\F}^2$ in our DANK model requires that the adaptive matrix $\bm F \in \mathcal{F}$ is expected to slightly vary around the all-one matrix, so each element in $[\bm \Gamma(\bm \alpha)]_{ij}$ cannot be significantly larger than $1$.
To this end, recall Eq.~\eqref{gammadef} and problem~\eqref{subF}, the regularizer parameter $\eta$ is chosen as $\eta \in \mathcal{O}(n)$ to ensure $\mathrm{tr}[\bm \Gamma(\bm \alpha)]$ to be in the same order with $\mathrm{tr}(\bm 1 \bm 1^{\!\top}) \in \mathcal{O}(n)$.

Lemma~\ref{fbound} gives the upper spectral bound of $\bm F^*$, which demonstrates that the feasible region $\mathcal{F}$ is a subset of the PSD cone $\mathcal{S}_+^n$.
In this case, we can directly solve $\bm F$ in the subset $\mathcal{F} \subseteq \mathcal{S}_+^n$ instead of the entire PSD cone $\mathcal{S}_+^n$, which makes it possible to seek for a good trade-off between the model flexibility and complexity.
Figure~\ref{flexiF} experimentally validates the effectiveness of the introduced constraints. We find that $\bm F^*$ is of low-rank and its entries range from 0.9 to 1.1 in a small region.
Therefore, such small fluctuation on $\bm F^*$ and its low-rank structure effectively control the model complexity.

\begin{figure}
	\centering
	\subfigure[$\sigma = 1$]{\label{ranksigma1}
		\includegraphics[width=0.35\textwidth]{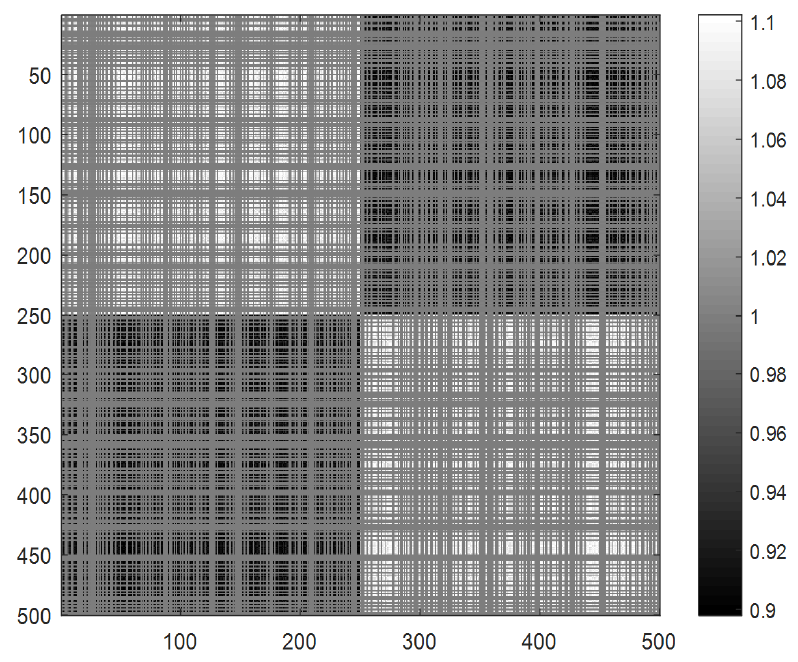}}
	\hspace{0.5cm}
	\subfigure[$\sigma=0.07$ by cross validation]{\label{ranksigma2}
		\includegraphics[width=0.35\textwidth]{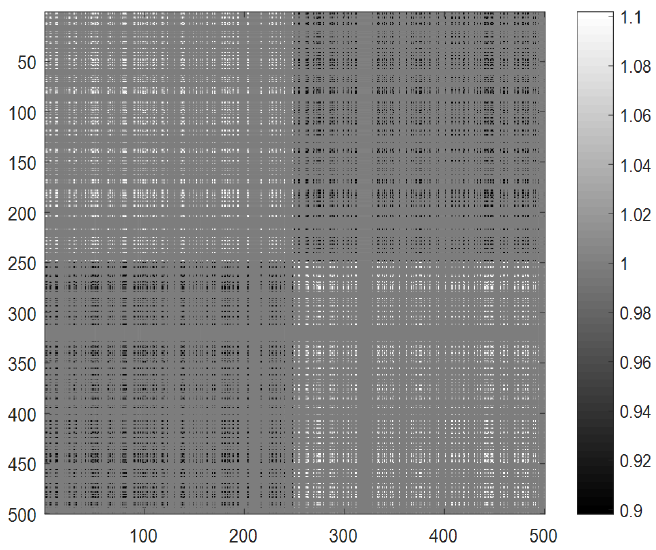}}
	\caption{Entries of $\bm F^*$ in our DANK model with the Gaussian kernel on the \emph{clowns} data set. }\label{flexiF}
\end{figure}


\subsubsection{Fast eigenvalue decay}
\label{sec:ana}
	In the above subsection, we have theoretically and experimentally validated the boundedness of the optimal solution $\bm F^*$, and thus this property is beneficial to control the model complexity.
	In this subsection, we elucidate that the introduced constraints/regularizers result in a fast eigenvalue decay of the learned kernel (matrix), which would be helpful to achieve good prediction performance.\footnote{It appears non-trivial to obtain rigorous analysis on generalization properties of non-parametric kernel learning as the learned kernel is sample-dependent and implicit.}
	
	Denote the learned kernel matrix as $\widetilde{\bm K} \equiv \widetilde{\bm K}(\bm \alpha) = \bm F \odot \bm K$, the learned kernel as $\tilde{k}$, and its associated RKHS as $\widetilde{\mathcal{H}}$. In fact, the eigenvalue decay of $\widetilde{\bm K}$ is leveraged to characterize the ``size” of $\widetilde{\mathcal{H}}$.
	For example, a fast eigenvalue decay of a kernel matrix implies that functions in the associated RKHS are smooth, and thus achieving a good prediction performance \citep{bach2013sharp}. 
	To this end, we need to control the complexity of $\widetilde{\mathcal{H}}$ for a fast eigenvalue decay of $\widetilde{\bm K}$. This can be achieved by the introduced two constraints/regularizers in our DANK model. \\
	i) the low-rank structure: According to $\mathrm{rank}(\widetilde{\bm K}) \leq \mathrm{rank}(\bm F) \mathrm{rank}(\bm K)$, see Theorem~3.2 in \citep{Styan1973Hadamard}, the rank of the learned kernel matrix $\widetilde{\bm K}$ mainly depends on the pre-given kernel matrix $\bm K$ as the introduced low-rank constraint on $\bm F$ yields $\mathrm{rank}(\bm F) \ll n$.
	Accordingly, we can obtain a low rank matrix $\widetilde{\bm K}$ if the pre-given kernel matrix $\bm K$ is of low rank.
	Here, the low rank property of $\bm K$, i.e., the finite non-zero eigenvalues of $\bm K$, implies a fast eigenvalue decay of $\bm K$, which holds by the following two common cases: (1) the geometric/exponential decay with $\lambda_i(\bm K) \propto nR_0 \mathrm{e}^{-a i}$ ($r_0$ and $a$ are some constants); (2) the polynomial decay with $\lambda_i(\bm K) \propto nR_0 i^{-b}$ with $b > 1$.
	Under these two cases, the learned kernel matrix $\widetilde{\bm K}$ is able to exhibit a fast eigenvalue decay, which would lead to a good prediction performance \citep{bach2013sharp}, or a tight estimation error bound \citep{liu2015eigenvalues}. \\
	ii) the bounded constraint: The introduced centering regularizer $\| \bm F - \bm 1 \bm 1^{\!\top} \|_{\mathrm{F}}^2$ would restrict $\bm F$ to vary around the all-one matrix in a small region, so $\bm F$ can be decomposed into $\bm F := \bm 1 \bm 1^{\!\top} + \bm E$, where $\bm E$ is regarded as a perturbation matrix with small residual error, i.e., $\| \bm E \|_{\mathrm{F}}$ is small.
	According to Hoffman-Wielandt inequality \citep{hoffman2003variation} in matrix perturbation theory, we have $\sum_{i=1}^n \left[ \lambda_i(\bm F) - \lambda_i(\bm 1 \bm 1^{\!\top}) \right]^2 \leq \| \bm E \|_{\mathrm{F}}^2$.
	Since the rank-one matrix $\bm 1 \bm 1^{\!\top}$ has only one non-zero eigenvalue with $\lambda_{1}(\bm 1 \bm 1^{\!\top})=n$, and its remaining eigenvalues are zero, we have $[\lambda_1(\bm F) - n]^2 + \sum_{i=2}^n [\lambda_i(\bm F)]^2 \leq \| \bm E \|_{\mathrm{F}}^2$.
	Roughly speaking, eigenvalues of $\bm F$ can be well approximated by a rank-one matrix $\bm 1 \bm 1^{\!\top}$ to some extent.
	That means, the centering regularizer could also bring the low-rank property on $\bm F$, which is useful to control the complexity of the solving space.
	Additionally, in the light of this, we omit the exact low-rank constraint/regularizer in the large-sample case due to its inseparable property. In this case, the learned $\bm F$ still exhibits a relative low-rank structure, see Section~\ref{sec:knplapp} for details.
	
	Based on the above analyses, we elucidate that, on one hand, the introduced low-rank regularizer and the centering regularizer ensure that $\bm F$ can be optimized in a bounded region in our unconstrained optimization problem.
	On the other hand, the used constraints/regularizers are beneficial to achieve a fast eigenvalue decay of the learned kernel (matrix) for good prediction performance.

\subsection{Out-of-sample Extension}
\label{sec:ose}
In our DANK model, the learned kernel (matrix) is non-parametric, so it is unknown for test data.
This is a so-called out-of-sample extension problem \citep{Bengio2004Out,fanuel2017positive,Pan2016Out}, which extensively exists in non-parametric kernel learning \citep{Lu2009Geometry,Zhuang2011Anonk,Liu2018AAAIme}, metric learning \citep{Kulis2013Metric,Jain2017Learning}, and nonlinear manifold learning \citep{Xie2013Onn}.
To address this issue, we develop a simple but effective technique, i.e., the reciprocal nearest neighbor scheme, to establish the learned kernel (matrix) for test data.

Given the optimal $\bm F^*$ on training data, the test data $\{ \bm{x}'_i \}_{i=1}^m$, and the initial kernel matrix for test data $\bm K' = [k(\bm x_i, \bm x'_j)]_{n \times m}$, we aim to establish the adaptive matrix $\bm F'$ for test data by the reciprocal nearest neighbor scheme.
Formally, we first construct the similarity matrix $\bm M$ between the training data and the test data based on the nearest neighbor scheme.
The used distance to find the nearest neighbor is the standard $\| \bm x_i - \bm x_j \|_2$ metric in the $d$-dimensional Euclidean space.
Assuming that $\bm x'_j$ is the $r$-th nearest neighbor of $\bm x_i$ in the set of $\{ \bm x'_t \}_{t=1}^m$, denoted as $\bm x'_j = \text{NN}_r(\bm x_i,\{ \bm x'_t \}_{t=1}^m)$.
Meanwhile $\bm{x}_i$ is the $s$-th nearest neighbor of $\bm{x}'_j$ in $\{ \bm x_t \}_{t=1}^n$, denoted as $\bm{x}_i =\text{NN}_s(\bm{x}'_j,\{ \bm x_t \}_{t=1}^n)$, then the similarity matrix $\bm M \in \mathbb{R}^{n \times m}$ is defined by
\begin{equation}\label{mBBPdef}
M_{ij} =\frac{1}{rs},\!~\text{if}~ \bm x'_j = \text{NN}_r\left(\bm x_i,\{ \bm x'_t \}_{t=1}^m\right)  \wedge
\bm{x}_i =\text{NN}_s\left(\bm{x}'_j,\{ \bm x_t \}_{t=1}^n\right),\quad \forall i \in [n], j \in [m] \,,
\end{equation}
and thus $\bm M_j = [M_{1j}, M_{2j}, \cdots, M_{nj}]^{\!\top} \in \mathbb{R}^n$ describes the relationship between $\bm x'_j$ and the training data $\{ \bm x_i \}_{i=1}^n$. 
This is a much stronger and more robust indicator of similarity than the simple and unidirectional nearest neighborhood relationship, since it takes into account the local densities of vectors around $\bm x_i$ and $\bm x'_j$. 
This reciprocal nearest scheme has been extensively applied to computer vision, such as image retrieval \citep{qin2011hello} and person re-identification \citep{zhong2017re,zheng2012reidentification}. 
Accordingly, $\bm F'$ is given by
\begin{equation*}
\bm{F}'_{j} \leftarrow  \bm{F}^*_{j^*},\!~\mbox{if}~ {j^*} = \argmax_{t}~ \{ M_{1j}, M_{2j}, \cdots, M_{tj}, \cdots, M_{nj} \} \, \quad \forall t \in [n]\,.
\end{equation*}
That is to say, if $\bm{x}_{j^*}$ is the ``optimal" reciprocal nearest neighbor of $\bm{x}'_j$ among the training data set $\{ \bm x_t \}_{t=1}^n$, the $j^*$-th column of $\bm{F}^*$ is assigned to the $j$-th column of $\bm{F}'$.
By doing so, $\bm F^*$ results in a flexible kernel matrix ${\bm F}' \odot {\bm{K}}' $ for test data.
Admittedly, we would be faced with the inconsistency if we directly extend the training kernel to the test kernel in this way.
However, in our model, $\bm F$ is designed to vary in a small range with low-rank structure, and is expected to smoothly vary between any two neighboring data points in $\mathcal{F}$, so the extension to $\bm F'$ on test data by this scheme is reasonable. Further, we attempt to provide some theoretical justification for this scheme as follows.

Mathematically, if $\bm x$ and $\bm x'$ are the reciprocal nearest neighbor pair, we expect that $\| \varphi^{\!\top}(\bm x) \varphi(\bm x'_j) - \varphi^{\!\top}(\bm x') \varphi(\bm x'_j) \|_2^2 $ is small on the test data $\{ \bm x'_j \}_{j=1}^m$, where $\varphi$ is the learned implicit feature mapping such that $F_{ij}K_{ij} = \langle \varphi(\bm x_i), \varphi(\bm x_j) \rangle_{\widetilde{\mathcal{H}}}$ on the training data.
\cite{balcan2006kernels} demonstrate that, in the presence of a large margin $\gamma$, a kernel function can also be viewed as a mapping from the input space $\mathcal{X}$ into an $\tilde{\mathcal{O}}(1/\gamma^2)$ space.
\begin{proposition}\label{prokernelmap} \citep{balcan2006kernels}
	Given $0 < \epsilon \leq 1$, the margin $\gamma$ and the implicit mapping $\varphi(\cdot)$ in SVM, then with probability at least $1-\delta$, let $d'$ be a positive integer such that $d' \geq d_0 = \mathcal{O}(\frac{1}{\gamma^2} \log \frac{1}{\epsilon \delta})$, for any $\bm x \in \{ \bm x_i \}_{i=1}^n \subset \mathbb{R}^d$ and $\bm x' \in \{ \bm x'_i \}_{j=1}^m \subset \mathbb{R}^d$ with $d \geq d'$, we have
	\begin{equation*}
	(1-\epsilon) \| \bm x - \bm x' \|_2^2 \leq \| \varphi(\bm x) - \varphi(\bm x') \|_2^2 \leq (1+\epsilon) \| \bm x - \bm x' \|_2^2\,,
	\end{equation*}
	where the mapping $\varphi$ is a random projection following with the Gaussian distribution or the uniform distribution.
\end{proposition}
{\bf Remark:} 
If the learned $\tilde{k}$ in our DANK model is sub-Gaussian, the above bounds can be achieved \citep{shi2012margin}. Note that the sub-Gaussian kernel assumption is mild as demonstrated by \cite{dao2017gaussian}.
Regarding to $d'$, we choose the lower bound $d = d' = d_0 = \mathcal{O}(\frac{1}{\gamma^2} \log \frac{1}{\epsilon \delta})$, which can be achieved by the margin $\gamma = {1}/{\| \bm w \|_2}$ and $\| \bm w \|_2^2 \in \mathcal{O}(d)$.

Based on the above analysis, given the ``optimal" reciprocal nearest neighbor pair $(\bm x, \bm x')$, for any test data point $\bm x'_j$ with $j \in [m]$, we have
\begin{equation}\label{propmap}
\begin{split}
\| \varphi^{\!\top}(\bm x) \varphi(\bm x'_j) - \varphi^{\!\top}(\bm x') \varphi(\bm x'_j) \|_2^2 & \leq \| \varphi(\bm x) - \varphi(\bm x') \|_2^2 \| \varphi^{\!\top}(\bm x'_j) \varphi(\bm x'_j) \|_2 \\
& \leq \underbrace{\| \varphi^{\!\top}(\bm x'_j) \varphi(\bm x'_j) \|_2}_{\mbox{effected by $\bm F \odot \bm K$ for $\bm x'_j$}} (1+\epsilon) \underbrace{\| \bm x - \bm x' \|_2^2}_{\mbox{small}} \\
& \leq \widetilde{C} (1+\epsilon) {\| \bm x - \bm x' \|_2^2}\,,
\end{split}
\end{equation}
where $\widetilde{C}$ is some constant as the learned kernel is bounded.
As a result, we can obtain a small $\| \varphi^{\!\top}(\bm x) \varphi(\bm x'_j) - \varphi^{\!\top}(\bm x') \varphi(\bm x'_j) \|_2^2$ if $\| \bm x - \bm x' \|_2^2$ is small, which is beneficial to out-of-sample extensions.


\subsection{Connections to Other Models}
Our non-parametric kernel learning framework in fact covers several typical models.

{\bf Kernel with multiple layers:} A kernel with $l$ layers in deep architectures \citep{cho2009kernel} is defined as
\begin{equation*}\label{layerkernel}
k^{(l)}(\bm x_i,\bm x_j) = \left \langle \phi^{(l)}\big( \cdots \phi^{(1)}(\bm x_i) \big), \phi^{(l)}\big( \cdots \phi^{(1)}(\bm x_j) \big) \right \rangle \,,
\end{equation*}
which computes the inner product between two data points $\bm x_i$ and $\bm x_j$ after $l$ successive applications of the nonlinear mapping $\phi(\cdot)$.
For example, the two layer composition of Gaussian kernel can be formulated as
\begin{equation}\label{equiv}
\begin{split}
&k^{(2)}(\bm x_i, \bm x_j)= \langle \phi^{(2)}(\phi^{(1)}(\bm x_i)), \phi^{(2)}(\phi^{(1)}(\bm x_j)) \rangle = e^{-\frac{1}{\sigma^2}}\exp \left(\frac{k(\bm x_i, \bm x_j)}{\sigma^2} \right) \\
\end{split}
\,.
\end{equation}
We observe that the nested kernel in Eq.~\eqref{equiv} can be decomposed into a fixed Gaussian kernel $k(\bm x_i, \bm x_j)$ and a nonlinear pairwise function $g(\bm x_i, \bm x_j)$ such that $k^{(2)}(\bm x_i, \bm x_j) := g(\bm x_i, \bm x_j)k(\bm x_i, \bm x_j)$, which is actually associated with the data-adaptive matrix $\bm F = [g(\bm x_i, \bm x_j)]_{n \times n}$.
That means, using a single kernel as well as a nonlinear pairwise layer could achieve a comparable and even better model flexibility when compared to the two-layer kernel framework. Learning the nonlinear pairwise function $g(\cdot, \cdot)$ by the matrix $\bm F$ via $n$ samples is an interpolation problem, which is also related to recently popular interpolation learning \citep{hastie2019surprises,bartlett2020benign}.
Further, recent deep kernel architectures \citep{Wilson2016Deep,mairal2016end,chen2020convolutional} bridge neural networks and kernel methods and achieve promising performance on various tasks. 
This will motivate us to design a multi-layer version of our DANK model, in which the parameters can be optimized in a layer-by-layer way in the training process. 
We leave this to future work.

{\bf Hyper-parameter learning:} In our model, the entry in the kernel matrix is learned from the data, and thus Gaussian kernels with flexible variances \citep{ying2007learnability} can be linked to our framework. 
Besides, based on the Sch\"{o}nberg's representation theorem \citep{wendland2004scattered}, we consider our DANK model in a distribution view \citep{khuzani2020mean}
\begin{equation}
\tilde{k}(\bm x_i, \bm x_j)=\int_{0}^{\infty} e^{-\xi \|\bm x_i - \bm x_j \|_{2}^{2}} \mu(\mathrm{d} \xi) \approx \frac{1}{D} \sum_{r=1}^D e^{-\xi_r \| \bm x_i - \bm x_j \|^2_2}\,,
\end{equation}
where $\mu$ is an implicit non-negative Borel measure with $\{ \xi_r \}_{r=1}^D \sim \mu(\cdot)$ that corresponds to the data-adaptive matrix $\bm F$. 

\section{Algorithm for DANK Model in SVM}
\label{sec:knplopt}
Extra-gradient based methods can be directly applied to solve the max-min problem~\eqref{mainrank}, and have been shown to exhibit an $\mathcal{O}(1/t)$ convergence rate \citep{nemirovski2004prox}, where $t$ is the iteration number.
Further, to accelerate the convergence rate, this section investigates the gradient-Lipschitz continuity of $h({\bm \alpha})$ in Eq.~\eqref{falpha}.
Based on this, we introduce the Nesterov's smooth optimization method \citep{Yu2005Smooth} that requires $\nabla h({\bm \alpha})$ Lipschitz continuous to solve problem~\eqref{mainrank}, that is shown to achieve $\mathcal{O}(1/t^2)$ convergence rate.

\subsection{Gradient-Lipschitz Continuity of DANK in SVM}
\label{sec:gradl}
To prove the gradient-Lipschitz continuity of $h({\bm \alpha})$, we need the following lemma.



\begin{lemma}\label{gammabound}
	Under the same condition in Lemma~\ref{fbound}, for any $\bm \alpha_1$, $\bm \alpha_2$ $\in \mathcal{A}$, we have
	\begin{equation*}
	\begin{split}
	\big\| \bm F(\bm \alpha_1) - \bm F(\bm \alpha_2)\big\|_{\F}
	\leq  \frac{\|\bm{K}\|_{\F}\big( \| {\bm \alpha_1} \|_2 + \| {\bm \alpha_2} \|_2 \big)}{4\eta} \big\|  {\bm \alpha_1} - {\bm \alpha_2} \big\|_{2}\,,
	\end{split}
	\end{equation*}
	where $\bm F(\bm \alpha_1) =\mathcal{J}_{\frac{\tau}{2}}\big(\bm 1 \bm{1}^{\!\top}+\bm \Gamma(\bm \alpha_1)\big)$ and $\bm F(\bm \alpha_2) =\mathcal{J}_{\frac{\tau}{2}}\big(\bm 1 \bm{1}^{\!\top}+\bm \Gamma(\bm \alpha_2)\big)$.
\end{lemma}
\begin{proof}
	The proofs can be found in Appendix~\ref{proofgammabound}.
\end{proof}
Formally, based on Lemmas~\ref{fbound} and \ref{gammabound}, we present the following theorem.
\begin{theorem}\label{theor}
	The function $h ({\bm \alpha})$ in Eq.~\eqref{falpha} is gradient-Lipschitz continuous, i.e.
	\begin{equation*}
	\| \nabla h(\bm \alpha_1) - \nabla h(\bm \alpha_2) \|_2 \leq L\| {\bm \alpha_1} - {\bm \alpha_2} \|_2,\quad \forall \bm \alpha_1, \bm \alpha_2 \in \mathcal{A}\,,
	\end{equation*}
	where the Lipschitz constant is $L=\kappa \left( n+{ 3nC^2 \|\bm{K}\|_{\F}}/{4\eta} \right)$ with $\kappa := \sup_{\bm x, \bm x' \in \mathcal{X}} |k(\bm x, \bm x')| $.
\end{theorem}
\begin{proof}
	The proofs can be found in Appendix~\ref{prooftheor}.
\end{proof}
The above theoretical analyses demonstrate that $\nabla h({\bm \alpha})$ is Lipschitz continuous, which provides a justification for utilizing a smooth optimization Nesterov's acceleration method to solve problem~\eqref{mainrank} with faster convergence.

\subsection{Nesterov's Smooth Optimization Method}
\label{sec:nest}
Here we introduce a projected gradient algorithm with Nesterov's acceleration to solve the optimization problem \eqref{mainrank}.
\cite{Yu2005Smooth} proposes an optimal scheme for smooth optimization
$\min_{\bm x \in \mathcal{Q}} g(\bm x)$,
where $g(\cdot)$ is a convex gradient-Lipschitz continuous function over a closed convex set $\mathcal{Q}$.
Introducing a continuous and strongly convex function denoted as \emph{proxy-function} $d(\bm x)$ on $\mathcal{Q}$, the first-order projected gradient method with Nesterov's acceleration can then be used to solve this problem.
In our model, we aim to solve the following convex problem
\begin{equation}\label{hopt}
\max\limits_{\bm \alpha \in \mathcal{A}} ~~h(\bm \alpha)\,,
\end{equation}
where $h(\bm \alpha)$ is concave and gradient-Lipschitz continuous with the Lipschitz constant $L$ in Theorem~\ref{theor}.
Here the proxy-function is defined as
$ d(\bm \alpha) = \frac{1}{2} \| \bm \alpha - \bm \alpha_0 \|_2^2 $
with $\bm \alpha_0 \in \mathcal{A}$.
The first-order Nesterov's smooth optimization method for solving problem \eqref{mainrank} is summarized in Algorithm~\ref{ago3}.

\begin{algorithm}[t]
	\caption{Projected gradient method with Nesterov's acceleration for problem \eqref{mainrank}}
	\label{ago3}
	\KwIn{The kernel matrix $\bm{K}$, the label matrix $\bm{Y}$, and the Lipschitz constant in Theorem~\ref{theor}}
	\KwOut{The optimal ${\bm \alpha}^*$}
	Set the stopping criteria $t_{\max}=2000$ and $\epsilon=10^{-4}$.\\
	Initialize $t = 0$ and ${\bm \alpha^{(0)}} \in \mathcal{A} := \bm 0$.\\
	\SetKwRepeat{RepeatUntil}{Repeat}{Until}
	\RepeatUntil{$t \geq t_{\max}$ or $ \| \bm \alpha^{(t)} - \bm \alpha^{(t-1)} \|_2 \leq \epsilon$}
	{ Compute $\bm{F}({\bm \alpha^{(t)}})=\mathcal{J}_{\frac{\tau}{2}}\big(\bm 1 \bm{1}^{\!\top}+\bm \Gamma(\bm \alpha^{(t)}) \big)$\;
		Compute $\nabla h ({\bm \alpha^{(t)}}) = \bm{1} - \bm{Y}\big(\bm{F}({\bm \alpha^{(t)}})\odot\bm{K}\big)\bm{Y} {\bm \alpha^{(t)}}$ \;
		Compute $\bm \theta^{(t)} \!=\! \mathcal{P}_{\mathcal{A}}\Big(\bm \alpha^{(t)} + \frac{1}{L}\nabla h ({\bm \alpha^{(t)}})\Big)$ \;
		Compute $\bm \beta^{(t)}\!=\!\mathcal{P}_{\!\mathcal{A}}\Big( \!\bm \alpha^{(0)} \!-\! \frac{1}{2L}\!\sum_{i=0}^{t}(i+1) \nabla h ({\bm \alpha^{(i)}}) \Big)$\;
		Set $\bm \alpha^{(t+1)} = \frac{t+1}{t+3}\bm \theta^{(t)} + \frac{2}{t+3}\bm \beta^{(t)}$\;
		$t := t + 1$\;}
\end{algorithm}

The key steps of Nesterov's acceleration are characterized by Lines 6, 7, and 8 in Algorithm~\ref{ago3}.
To be specific, according to \cite{Yu2005Smooth}, at the $t$-th iteration, we need to solve the following problem
\begin{equation}\label{nest1}
\bm \beta^{(t)} = \argmin_{\bm \alpha \in \mathcal{A}} \frac{L}{2} \| \bm \alpha - \bm \alpha_0 \|_2^2 +  \sum_{i=0}^{t} \frac{i+1}{2} \Big[ h(\bm \alpha^{(i)}) + \big\langle \nabla h ({\bm \alpha^{(i)}}), \bm \alpha - \bm \alpha^{(i)} \big\rangle \Big]\,,
\end{equation}
which is equivalent to
\begin{equation*}
\bm \beta^{(t)} = \argmin_{\bm \alpha \in \mathcal{A}} \| \bm \alpha - \bm \alpha_0 \|_2^2 + \frac{2}{L} \sum_{i=0}^{t} \frac{i+1}{2} \nabla h^{\!\top} ({\bm \alpha^{(i)}}) \bm \alpha \,,
\end{equation*}
where we omit the irrelevant terms $h(\bm \alpha^{(i)})$ and $\nabla h^{\!\top} ({\bm \alpha^{(i)}}) \bm \alpha^{(i)}$ that are independent of the optimization variable $\bm \alpha$ in Eq.~\eqref{nest1}.
Accordingly, the above problem is further reformulated as
\begin{equation*}
\bm \beta^{(t)} = \argmin_{\bm \alpha \in \mathcal{A}}~~ \Big\| \bm \alpha - \bm \alpha_0 +  \frac{1}{2L}\!\sum_{i=0}^{t}(i+1) \nabla h ({\bm \alpha^{(i)}}) \Big\|_2^2\,,
\end{equation*}
of which the optimal solution is $\bm \beta^{(t)} = \mathcal{P}_{\!\mathcal{A}}\Big( \!\bm \alpha_{0} \!-\! \frac{1}{2L}\!\sum_{i=0}^{t}(i+1) \nabla h ({\bm \alpha^{(i)}}) \Big)$ as outlined in Line 7 in Algorithm~\ref{ago3}, where $\mathcal{P}_{\mathcal{A}}(\bm \alpha)$ is a projection operator that projects $\bm \alpha$ over the set $\mathcal{A}$.
A quick note on projection onto the feasible set $\mathcal{A}=\{ {\bm \alpha}\in \mathbb{R}^n: {\bm \alpha}^{\top}\bm{y}=0,~\bm 0\leq {\bm \alpha} \leq C\bm{1} \}$: it typically suffices in practice to use the  alternating projection algorithm \citep{von1949rings}. 
Since the feasible set $\mathcal{A}$ is the intersection of a hyperplane and a hypercube, both of them admit a simple projection step. To be specific, first clip $\bm \alpha$ to $[0,C]$, and then project on the hyperplane $\bm \alpha \leftarrow \bm \alpha - \frac{\bm y^{\!\top} \bm \alpha}{n}\bm y$.
The convergence rate of the alternating projection algorithm is shown to be linear \citep{von1949rings} and thus it is very efficient.

It can be noticed that when Lines 6, 7, and 8 in Algorithm~\ref{ago3} are replaced by
\begin{equation}\label{pgm}
\bm \alpha^{(t+1)} \!=\! \mathcal{P}_{\mathcal{A}}\Big(\bm \alpha^{(t)} + \frac{1}{L}\nabla h ({\bm \alpha^{(t)}})\Big)\,
\end{equation}
with the Lipschitz constant $L=n - \frac{\tau}{2}+\frac{nC^2}{4\eta}\lambda_{\max}(\bm K)$ derived from Lemma~\ref{fbound}, the Nesterov's smooth method degenerates to a standard projected gradient method.
The convergence of the Nesterov smoothing optimization algorithm is pointed out by Theorem 2 in \citep{Yu2005Smooth}
\begin{equation*}
h(\bm \alpha^*) - h(\bm \beta^{(t)}) \leq \frac{8L\| \bm \alpha_0 -\bm \alpha^*\|^2}{(t+1)(t+2)}\,,
\end{equation*}
where $\bm \alpha^*$ is the optimal solution of Eq.~\eqref{hopt}.
Note that, in general, Algorithm~\ref{ago3} cannot guarantee $\{ h(\bm \alpha^{(t)}): t \in \mathbb{N} \}$ and $\{ h(\bm \beta^{(t)}): t \in \mathbb{N} \}$ to be monotonely increasing during the maximization process.
Nevertheless, such algorithm can be modified to obtain a monotone sequence with replacing Line 6 in Algorithm~\ref{ago3} by
\begin{align*}
\left\{
\begin{array}{rcl}
\begin{split}
& \tilde{\bm \theta}^{(t)} = \mathcal{P}_{\!\mathcal{A}}\Big(\bm \alpha^{(t)} + \frac{1}{L}\nabla h ({\bm \alpha^{(t)}})\Big)\,, \\
& \\
& \bm {\theta}^{(t)} = \argmax_{\bm \alpha} h(\bm \alpha),~\bm \alpha \in \{ \bm {\theta}^{(t-1)},  \tilde{\bm \theta}^{(t)}, \bm \alpha^{(t)} \}\,.
\end{split}
\end{array} \right.
\end{align*}

The Nesterov's smooth optimization method takes $\mathcal{O}(\sqrt{L/\epsilon})$ to find an $\epsilon$-optimal solution, which is better than the standard projected gradient method with the complexity $\mathcal{O}({L/\epsilon})$.


\section{DANK in Large Scale Case}
\label{sec:knplapp}
Scalability in kernel methods is a vital issue which often limits their applications in large data sets \citep{rahimi2007random,wang2016spsd,liu2020survey}, especially for nonparametric kernel learning optimized by semi-definite programming. 
Hence, in this section, we take our DANK model embedded in SVM as an example to study our kernel approximation approach.
The presented results in this section are also suitable to other nonparametric kernel learning based algorithms.

To consider the scalability of our DANK model embedded in SVM in large-scale situations, problem~\eqref{mainrank} is reformulate as
\begin{equation}\label{mainls}
\begin{split}
&\max\limits_{\bm \alpha}\min \limits_{\bm{F}\in \mathcal{S}^n_+} ~~ H(\bm \alpha, \bm F) = \bm{1}^{\!\top}\!{\bm \alpha}\! -\!\frac{1}{2} {\bm \alpha}^{\!\!\top}\bm{Y}\big(\bm{F}\odot\bm{K}\big)\bm{Y} {\bm \alpha} + {\eta} \| \bm{F} - \bm{1} \bm{1}^{\top} \|_{\text{F}}^2 \\
&\mbox{s.t.}  ~~\bm{0}\leq {\bm \alpha} \leq C\bm{1}\,.
\end{split}
\end{equation}
where the bias term $b$ is usually omitted in the large scale issue \citep{keerthi2006building,hsieh2014divide,lian2017divide}.
Besides, we have to omit the low-rank regularizer on $\bm F$ due to its inseparable property, which is reasonable based on the rapid decaying spectra of the kernel matrix \citep{smola2000sparse}.
Specifically, in Section~\ref{sec:ana}, we have demonstrated that $\bm F$ would exhibit the low-rank property as well by the centering regularizer $\| \bm{F} - \bm{1} \bm{1}^{\top} \|_{\text{F}}^2$.
Furthermore, we will experimentally verify that dropping the low-rank term in large-scale problems has no much sacrifice for the accuracy in Section~\ref{sec:explowrank}.

In our decomposition-based scalable approach, we divide the data into small subsets by k-means, and then solve each subset independently and efficiently.
Such similar scheme also exists in \citep{hsieh2014divide,zhang2013divide,si2017memory}.
To be specific, we firstly partition the data into $v$ subsets $\{ \mathcal{V}_1, \mathcal{V}_2,\dots, \mathcal{V}_v \}$, and then solve the respective sub-problems independently with the following formulation
\begin{equation}\label{mainlssub}
\begin{split}
&\max\limits_{\bm \alpha^{(c)}}\!\min \limits_{\bm{F}^{(c,c)} \in \mathcal{S}^{|\mathcal{V}_c|}_+}  \bm{1}^{\!\top}\!{\bm \alpha^{(c)}}\!
\!+\! {\eta} \| \bm{F}^{(c,c)} - \bm{1} \bm{1}^{\top} \|_{\text{F}}^2
-\!\frac{1}{2} {\bm \alpha^{(c)}}^{\!\!\top}\bm{Y}^{(c,c)}\big(\bm{F}^{(c,c)}\odot\bm{K}^{(c,c)}\big) \bm{Y}^{(c,c)} {\bm \alpha}^{(c)}\\
&\mbox{s.t.}  ~~\bm{0}\leq {\bm \alpha}^{(c)} \leq C\bm{1},~~\forall~c=1,2,\dots,v\,,
\end{split}
\end{equation}
where $|\mathcal{V}_c|$ denotes the number of data points in $\mathcal{V}_c$.
Suppose that $(\bar{\bm \alpha}^{(c)}, \bar{\bm{F}}^{(c,c)})$ is the optimal solution of the $c$-th subproblem, the approximation solution $(\bar{\bm \alpha}, \bar{\bm F})$ to the whole problem is concatenated by $\bar{\bm \alpha}=[\bar{\bm \alpha}^{(1)}, \bar{\bm \alpha}^{(2)},\dots,\bar{\bm \alpha}^{(v)}]$ and $\bar{\bm F} = \diag(\bar{\bm{F}}^{(1,1)}, \bar{\bm{F}}^{(2,2)}, \dots, \bar{\bm{F}}^{(v,v)})$, where $\bar{\bm F}$ is a block-diagonal matrix.
In the next, we study the decomposition-based scalable approach in the following two aspects.
First, the objective function value $H(\bar{\bm \alpha},\!\bar{\bm{F}})$ in Eq.~\eqref{mainls} is close to $H(\bm \alpha^*, \bm F^*)$.
Second, if $\bm x_i$ is not a support vector of the subproblem, it will also be a non-support vector of the whole problem under some conditions.
To prove the above three propositions, we need the following lemma that links the subproblems to the whole problem.
\begin{lemma}\label{lemmalink}
	Given the optimal solution $(\bar{\bm \alpha}^{(c)}, \bar{\bm{F}}^{(c,c)})$ of problem~\eqref{mainlssub} with $c \in \{1, 2, \cdots, v \}$, by concatenating $\bar{\bm \alpha}=[\bar{\bm \alpha}^{(1)}, \bar{\bm \alpha}^{(2)},\dots,\bar{\bm \alpha}^{(v)}]$ and $\bar{\bm F} = \diag(\bar{\bm{F}}^{(1,1)}, \bar{\bm{F}}^{(2,2)}, \dots, \bar{\bm{F}}^{(v,v)})$, the approximation solution $(\bar{\bm \alpha}, \bar{\bm F})$ to the whole problem is the optimal solution of the following problem
	\begin{equation}\label{mainlslemma}
	\begin{split}
	&\max\limits_{\bm \alpha}\!\min \limits_{\bm{F}\in \mathcal{S}^n_+}~~\bar{H}(\bm \alpha, \bm F) \triangleq \bm{1}^{\!\top}\!{\bm \alpha} -\frac{1}{2} {\bm \alpha}^{\!\!\top}\bm{Y}\big(\bm{F}\odot \bar{\bm{K}} \big)\bm{Y} {\bm \alpha} + {\eta} \| \bm{F} - \bm{1} \bm{1}^{\top} \|_{\emph{F}}^2 \\
	&\emph{s.t.} ~~\bm{0}\leq {\bm \alpha} \leq C\bm{1}\,,
	\end{split}
	\end{equation}
	with the kernel $\bar{\bm K}$ defined by
	\begin{equation*}
	\bar{K}_{ij} = I(\pi(\bm x_i), \pi(\bm x_j)) K_{ij}\,,
	\end{equation*}
	where $\pi(\bm x_i)$ is the cluster that $\bm x_i$ belongs to, and $I(a,b)=1$ iff $a=b$, and $I(a,b)=0$ otherwise.
\end{lemma}
\begin{proof}
	The proofs can be found in Appendix~\ref{prooflemmalink}.
\end{proof}
Based on the above lemma, we are ready to investigate the difference between $H(\bm \alpha^*, \bm F^*)$ and $H(\bar{\bm \alpha},\!\bar{\bm{F}})$ as follows.
\begin{theorem}\label{theorappf}
	Denote $(\bm \alpha^*, \bm F^*)$ and $(\bar{\bm \alpha}, \bar{\bm F})$ as the optimal solutions of problem~\eqref{mainrank} and problem~\eqref{mainlslemma}, respectively. Suppose that each element in $\bm F^*$ and $\bar{\bm F}$ satisfies $0 < B_1 \leq \max\{ F^*_{ij}, \bar{F}_{ij} \} \leq B_2$, with $B=B_2 - B_1$, we have
	\begin{equation*}
	\left| H(\bm \alpha^*, \bm F^*) - H(\bar{\bm \alpha},\!\bar{\bm{F}}) \right| \leq \frac{1}{2}BC^2 Q(\pi)\,,
	\end{equation*}
	with $Q(\pi) = \sum_{i,j:\pi(\bm x_i) \neq \pi(\bm x_j) }^{n} | k(\bm x_i, \bm x_j)|$, where $\{ \pi(\bm x_1), \pi(\bm x_2), \cdots, \pi(\bm x_n) \}$ is the partition indicator and $C$ is the balance parameter in SVM.
\end{theorem}
\begin{proof}
	The proofs can be found in Appendix~\ref{prooftheorappf}.
\end{proof}
{\bf Remark: } $Q(\pi)$ actually consists of the off-diagonal values of the kernel matrix $\bm K$ if we rearrange the training data in a clustering order.
It depends on the data distribution, the number of clusters $v$, and the kernel type.
Intuitively, if the clusters are nicely shaped (e.g. Gaussian) and well-separated, then the kernel matrix may be approximately block-diagonal. In this case, $Q(\pi)$ would be small.
Let we examine two extreme cases of $v$. If $v=1$, i.e., only one cluster, we have $Q(\pi)=0$; If $v=n$, i.e., each data point is grouped into a cluster, then $Q(\pi)$ can be upper bounded by $Q(\pi) \leq \sum_{i,j}^n |k(\bm x_i, \bm x_j)| $.
In practical clustering algorithms, $v$ is often chosen to be much smaller than $n$, i.e., $v \ll n$, and thus we can obtain a small $Q(\pi)$. 
In fact, it appears non-trivial to quantitatively analyze the relationship between $Q(\pi)$ and $v$, so we experimentally study this relationship in Section~\ref{sec:bound}.

Besides, in SVM, we also concern about the relationship of support/non-support vectors between the subproblems and the whole problem.
Accordingly, we present the following theorem to explain this issue.
\begin{theorem}\label{theorsv}
	Under the same condition of~Theorem~\ref{theorappf} with an additional bounded assumption $\kappa := \sup_{\bm x, \bm x' \in \mathcal{X}}|k(\bm x, \bm x')| $, suppose that $\bm x_i$ is not a support vector of the subproblem, i.e., $\bar{\alpha}_i =0$, $\bm x_i$ will also not be a support vector of the whole problem i.e., ${\alpha}_i =0$, under the following condition
	\begin{equation}\label{gradsvc}
	\Big(\! \nabla_{\bm \alpha} \bar{H}(\bar{\bm \alpha}, \bar{\bm F}) \!\Big)_{\!i} \!\leq\! - (B \!+\! B_2)C \left(\| \bar{\bm K}_i \|_1 \!+\! \kappa \right) \leq - (B + B_2)C \left(\| \bm K \|_1 + \kappa \right) \,,
	\end{equation}
	where $\bar{\bm K}_i$ denotes the $i$-th column of the kernel matrix $\bar{\bm K}$.
\end{theorem}
\begin{proof}
	The proofs can be found in Appendix~\ref{prooftheorsv}.
\end{proof}
{\bf Remark:} Eq.~\eqref{gradsvc} is a sufficient condition and can be expressed as
\begin{equation*}
1 - nB_2 \kappa C \leq \Big(\! \nabla_{\bm \alpha} \bar{H}(\bar{\bm \alpha}, \bar{\bm F}) \!\Big)_{\!i} \leq  - (B + B_2)C \left(\| \bar{\bm K}_i \|_1 + \kappa \right) \,,
\end{equation*}
where the first inequality admits
$\inf \Big(\! \nabla_{\bm \alpha} \bar{H}(\bar{\bm \alpha}, \bar{\bm F}) \!\Big)_{\!i} = \inf \left( 1 -  \sum_{j=1}^{n} y_i y_j \bar{ F}_{ij} \bar{ K}_{ij} \bar{ \alpha}_j \right) = 1 - nB_2 \kappa C$.
So if we assume that $\left( \nabla_{\bm \alpha} \bar{H}(\bar{\bm \alpha}, \bar{\bm F}) \right)_{\!i}$ of non-support vectors is uniformly distributed over the range $[1-n\kappa B_2 C, 0]$, then nearly $1 - \frac{(B + B_2)C \left(\| \bar{\bm K}_i \|_1 + \kappa \right)}{nB_2 \kappa C - 1} \approx 1 - \frac{c}{\sqrt{n}}$ of the total non-support vectors can be directly recognized, where $c$ is some constant.
That means, the screening proportion of non-support vectors is $(1 - {c}/{\sqrt{n}})*100\%$, at a certain $\mathcal{O}(n^{-1/2})$ rate. 
If $\left( \nabla_{\bm \alpha} \bar{H}(\bar{\bm \alpha}, \bar{\bm F}) \right)_{\!i}$ follows with some heavy-tailed distributions over the range $[1-n\kappa B_2 C, 0]$, the recognized rate would decrease.
And specifically, we will experimentally check that our screening condition is reasonable in Section~\ref{sec:bound}.

\section{DANK Model in SVR}
\label{sec:knplsvr}
In this section, we incorporate the DANK model into SVR and also develop the Nesterov's smooth optimization algorithm to solve it.
Here the label space is $\mathcal{Y} \subseteq \mathbb{R}$ for regression.

Similar to DANK in SVM revealed by problem~\eqref{mainrank}, we incorporate the DANK model into SVR with the $\varepsilon$-insensitive loss, namely
\begin{equation}\label{mainsvr}
\begin{split}
& \max\limits_{\hat{\bm \alpha}, \check{\bm \alpha}}\min\limits_{{\bm F} \in \mathcal{S}^n_+} -\frac{1}{2}(\hat{\bm \alpha} - \check{\bm \alpha})^{\top} \big({\bm F} \odot \bm{K}\big)(\hat{\bm \alpha} - \check{\bm \alpha})  + (\hat{\bm \alpha} - \check{\bm \alpha})^{\top}\bm{y} \\
&\qquad \qquad ~- \varepsilon (\hat{\bm \alpha} + \check{\bm \alpha})^{\top}\bm{1} +\! {\eta} \| {\bm{F}} \!-\! \bm{1}\bm{1}^{\!\!\top} \! \|_{\text{F}}^2\! +\!\tau \eta \| {\bm{F}} \|_*\\
&\mbox{s.t.} ~0 \leq \hat{\bm \alpha}, \check{\bm \alpha} \leq C, ~(\hat{\bm \alpha} - \check{\bm \alpha})^{\top}\bm{y} = 0\,,
\end{split}
\end{equation}
where the dual variable is $\bm \alpha = \hat{\bm \alpha} - \check{\bm \alpha}$.
The objective function in problem~\eqref{mainsvr} is denoted as $H(\hat{\bm \alpha}, \check{\bm \alpha},\bm{F})$. Further, we define the following function
\begin{equation}\label{falphasvr}
{h}(\hat{\bm \alpha}, \check{\bm \alpha}) \triangleq H(\hat{\bm \alpha}, \check{\bm \alpha},\bm{F}^*)  = \min \limits_{\bm{F}\in \mathcal{S}^n_+} {H}(\hat{\bm \alpha}, \check{\bm \alpha},\bm{F})\,,
\end{equation}
where ${h}(\hat{\bm \alpha}, \check{\bm \alpha})$ can be obtained by solving the following problem
\begin{equation}\label{svrfa}
\min \limits_{{\bm F}\in \mathcal{S}^n_+} \| {\bm F} - \bm 1 \bm{1}^{\!\top} - {\bm \Gamma}(\hat{\bm \alpha}, \check{\bm \alpha})\|_{\F}^2 + \tau \| {\bm F} \|_*\,,
\end{equation}
with ${\bm \Gamma}(\hat{\bm \alpha}, \check{\bm \alpha})=\frac{1}{4\eta}\diag(\hat{\bm \alpha} - \check{\bm \alpha})^{\top}\bm{K}\diag(\hat{\bm \alpha} - \check{\bm \alpha})$.
The optimal solution of Eq.~\eqref{svrfa} is ${\bm F}^* =  \mathcal{J}_{\frac{\tau}{2}}(\bm 1 \bm{1}^{\!\top}+{\bm \Gamma(\hat{\bm \alpha}, \check{\bm \alpha})})$.
We can easily check that Lemma \ref{fbound} is also applicable to problem~\eqref{falphasvr}
\begin{equation*}
h(\hat{\bm \alpha}, \check{\bm \alpha}) = \min \limits_{{\bm{F}}\in \mathcal{B}} H(\hat{\bm \alpha}, \check{\bm \alpha},{\bm{F}})\,.
\end{equation*}

Similar to Lemma \ref{gammabound}, in our DANK model embedded in SVR, $\| {\bm \Gamma}(\hat{\bm \alpha}_1, \check{\bm \alpha}_1) - {\bm \Gamma}(\hat{\bm \alpha}_2, \check{\bm \alpha}_2)\|_2$ can be bounded by the following lemma.

\begin{lemma}\label{gammaboundsvr}
	For any $\hat{\bm \alpha}_1, \check{\bm \alpha}_1, \hat{\bm \alpha}_2, \check{\bm \alpha}_2 \in \mathcal{A}$, we have
	\begin{equation*}
	\begin{split}
	&\big\| {\bm F}(\hat{\bm \alpha}_1, \check{\bm \alpha}_1) -  {\bm F}(\hat{\bm \alpha}_2, \check{\bm \alpha}_2)\big\|_{\F}
	\leq  \| {\bm \Gamma}(\hat{\bm \alpha}_1, \check{\bm \alpha}_1) - {\bm \Gamma}(\hat{\bm \alpha}_2, \check{\bm \alpha}_2)\|_{\F} \\
	& \leq \frac{\|\bm{K}\|}{4\eta} \big \|\hat{\bm \alpha}_1-\check{\bm \alpha}_1+\hat{\bm \alpha}_2 - \check{\bm \alpha}_2 \big\|_2 \big\| \hat{\bm \alpha}_1 - \check{\bm \alpha}_1 - \hat{\bm \alpha}_2 + \check{\bm \alpha}_2\big\|_2\,,
	\end{split}
	\end{equation*}
	where $\bm{F}(\hat{\bm \alpha}_1, \check{\bm \alpha}_1)) =\mathcal{J}_{\frac{\tau}{2}}\big(\bm 1 \bm{1}^{\!\top}+ {\bm \Gamma}(\hat{\bm \alpha}_1, \check{\bm \alpha}_1)\big)$ and $\bm{F}(\hat{\bm \alpha}_2, \check{\bm \alpha}_2) =\mathcal{J}_{\frac{\tau}{2}}\big(\bm 1 \bm{1}^{\!\top}+ {\bm \Gamma}(\hat{\bm \alpha}_2, \check{\bm \alpha}_2)\big)$.
\end{lemma}
The proof of Lemma~\ref{gammaboundsvr} is similar to that of Lemma~\ref{gammabound}, and here we omit the detailed proof.
Next we present the partial derivative of ${h}(\hat{\bm \alpha}, \check{\bm \alpha})$ regarding to $\hat{\bm \alpha}$ and $\check{\bm \alpha}$.
\begin{proposition}\label{theorgradsvr}
	The objective function ${h}(\hat{\bm \alpha}, \check{\bm \alpha})$ with two variables defined by Eq.~\eqref{falphasvr} is differentiable and its partial derivatives are given by
	\begin{align}\label{gradfsvr}
	\left\{
	\begin{array}{rcl}
	\begin{split}
	& \frac{\partial {{{h}(\hat{\bm \alpha}, \check{\bm \alpha})}}}{\partial \hat{\bm \alpha}}  =- \varepsilon \bm I - ( \hat{\bm \alpha} - \check{\bm \alpha}) \bm{F} \odot \bm K + \bm y\,,  \\
	& \\
	& \frac{\partial {{{h}(\hat{\bm \alpha}, \check{\bm \alpha})}}}{\partial \check{\bm \alpha}}  =- \varepsilon \bm I - ( \hat{\bm \alpha} - \check{\bm \alpha}) \bm{F} \odot \bm K - \bm y\,.
	\end{split}
	\end{array} \right.
	\end{align}
\end{proposition}
Formally, ${h}(\hat{\bm \alpha}, \check{\bm \alpha})$ is proven to be gradient-Lipschitz continuous by the following theorem.
\begin{theorem}\label{theorsvr}
	The function ${h}(\hat{\bm \alpha}, \check{\bm \alpha})$ with its partial derivatives in Eq.~\eqref{gradfsvr} is gradient-Lipschitz continuous, i.e., for any $\hat{\bm \alpha}_1, \check{\bm \alpha}_1, \hat{\bm \alpha}_2, \check{\bm \alpha}_2 \in \mathcal{A}$, let the concentration vectors be $\tilde{\bm \alpha}_1 = [\hat{\bm \alpha}_1^{\!\top}, \check{\bm \alpha}_1^{\!\top}]^{\!\top}$ and $\tilde{\bm \alpha}_2 = [\hat{\bm \alpha}_2^{\!\top}, \check{\bm \alpha}_2^{\!\top}]^{\!\top}$, and the partial  derivatives be
	\begin{small}
		\begin{equation*}
		\begin{split}
		\nabla_{\tilde{\bm \alpha}_1} h(\hat{\bm \alpha}_1, \check{\bm \alpha}_1) & \!=\!\bigg[\Big( \frac{\partial {{{h}(\hat{\bm \alpha}, \check{\bm \alpha}_1)}}}{\partial \hat{\bm \alpha}}\big|_{\hat{\bm \alpha}\! =\! \hat{\bm \alpha}_1} \Big)^{\!\top}, \Big( \frac{\partial {{{h}(\hat{\bm \alpha}_1, \check{\bm \alpha})}}}{\partial \check{\bm \alpha}}\big|_{\check{\bm \alpha}\! =\! \check{\bm \alpha}_1} \Big)^{\!\top} \bigg]^{\!\top}\,, \\
		\nabla_{\tilde{\bm \alpha}_2} h(\hat{\bm \alpha}_2, \check{\bm \alpha}_2) & \!=\!\bigg[\Big( \frac{\partial {{{h}(\hat{\bm \alpha}, \check{\bm \alpha}_2)}}}{\partial \hat{\bm \alpha}}\big|_{\hat{\bm \alpha} \!=\! \hat{\bm \alpha}_2} \Big)^{\!\top}, \Big( \frac{\partial {{{h}(\hat{\bm \alpha}_2, \check{\bm \alpha})}}}{\partial \check{\bm \alpha}}\big|_{\check{\bm \alpha} = \check{\bm \alpha}_2} \Big)^{\!\top} \bigg]^{\!\top}\,, \\
		\end{split}
		\end{equation*}
	\end{small}
	we have
	\begin{small}
		\begin{equation*}
		\begin{split}
		& \|\nabla_{ \!\tilde{\bm \alpha}_1}  {h}(\hat{\bm \alpha}_1,  \check{\bm \alpha}_1 \!)  -  \nabla_{ \tilde{\bm \alpha}_2}  {h}( \hat{\bm \alpha}_2, \! \check{\bm \alpha}_2 ) \|_2  \!\leq  \! 2L  \Big(  \| \hat{\bm \alpha}_2  -  \hat{\bm \alpha}_1 \|_2  +  \| \check{\bm \alpha}_2  -  \check{\bm \alpha}_1  \|_2  \! \Big)\,.
		\end{split}
		\end{equation*}
	\end{small}
	where the Lipschitz constant is $L=2\kappa \Big(n+\frac{ 9nC^2 \|\bm{K}\|_{\F}}{4\eta} \Big)$.
\end{theorem}

%
\begin{proof}
	The proofs can be found in Appendix~\ref{prooftheorsvr}.
\end{proof}

Based on the gradient-Lipschitz continuity of ${h}(\hat{\bm \alpha}, \check{\bm \alpha})$ demonstrated by Theorem~\ref{theorsvr},
we are ready to present the first-order Nesterov's smooth optimization method for problem \eqref{mainsvr}.
The smooth optimization algorithm is summarized in Algorithm~\ref{ago3svr}.
\begin{algorithm}
	\caption{Projected gradient method with Nesterov's acceleration for problem \eqref{mainsvr}}
	\label{ago3svr}
	\KwIn{The kernel matrix $\bm{K}$, the label matrix $\bm{Y}$, and the Lipschitz constant $L$ derved in Theorem~\ref{theorsvr}}
	\KwOut{The optimal ${\bm \alpha}^*$}
	Set the stopping criteria $t_{\max}=2000$ and $\epsilon=10^{-4}$.\\
	Initialize $t = 0$ and ${\hat{\bm \alpha}^{(0)}}, {\check{\bm \alpha}^{(0)}} \in \mathcal{A} := \bm 0$.\\
	\SetKwRepeat{RepeatUntil}{Repeat}{Until}
	\RepeatUntil{$t \geq t_{\max}$ or $ \| \bm \alpha^{(t)} - \bm \alpha^{(t-1)} \|_2 \leq \epsilon$}
	{ Compute $\bm{F}({\hat{\bm \alpha}^{(t)}}, {\check{\bm \alpha}^{(t)}} )=\mathcal{J}_{\frac{\tau}{2}}\big(\bm 1 \bm{1}^{\!\top}+\bm \Gamma({\hat{\bm \alpha}^{(t)}}, {\check{\bm \alpha}^{(t)}} ) \big)$\;
		Compute  ${\partial {{h}}}/{\partial \hat{\bm \alpha}} $ and ${\partial {{h}}}/{\partial \check{\bm \alpha}}$ by Eq.~\eqref{gradfsvr}, and concentrate them as $\nabla h ({\hat{\bm \alpha}^{(t)}}, {\check{\bm \alpha}^{(t)}} ) = [({\partial {{h}}}/{\partial \hat{\bm \alpha}})^{\!\top}, ({\partial {{h}}}/{\partial \check{\bm \alpha}})^{\!\top}]^{\!\top}$ \;
		Compute $\bm \theta^{(t)} \!=\! \mathcal{P}_{\mathcal{A}}\Big([{\hat{\bm \alpha}^{(t)\!\top}}, {\check{\bm \alpha}^{(t)\!\top}}]^{\!\top} + \frac{1}{2L}\nabla h ({\hat{\bm \alpha}^{(t)}}, {\check{\bm \alpha}^{(t)}} ) \Big)$ \;
		Compute $\bm \beta^{(t)}\!=\!\mathcal{P}_{\!\mathcal{A}}\Big( \![\hat{\bm \alpha}^{(0)\!\top}, \check{\bm \alpha}^{(0)\!\top}]^{\!\top} \!-\! \frac{1}{4L}\!\sum_{i=0}^{t}(i+1) \nabla h ({\hat{\bm \alpha}^{(i)}}, {\check{\bm \alpha}^{(i)}} ) \Big)$\;
		Set $[{\hat{\bm \alpha}^{(t+1)\!\top}}, {\check{\bm \alpha}^{(t+1)\!\top}}]^{\!\top}  = \frac{t+1}{t+3}\bm \theta^{(t)} + \frac{2}{t+3}\bm \beta^{(t)}$\;
		Set $\bm \alpha^{(t+1)} = \hat{\bm \alpha}^{(t+1)} - \check{\bm \alpha}^{(t+1)}$ and $t := t + 1$\;}
\end{algorithm}

\section{Experimental Results}
\label{sec:experiment}
This section evaluates the performance of our DANK model in comparison with several representative kernel learning algorithms on classification and regression benchmark data sets.
All the experiments implemented in MATLAB are conducted on a Workstation with an Intel$^\circledR$ Xeon$^\circledR$ E5-2695 CPU (2.30 GHz) and 64GB RAM.
The source code of our DANK model in Algorithm~\ref{ago3} can be found in \url{http://www.lfhsgre.org}.

\subsection{Classification Tasks}

We conduct experiments on the UCI Machine
Learning Repository with small scale data sets,~\footnote{\url{https://archive.ics.uci.edu/ml/datasets.html}} and 
three large data sets including \emph{EEG},~\emph{ijcnn1} and \emph{covtype}.\footnote{All datasets are available at
\url{https://www.csie.ntu.edu.tw/~cjlin/libsvmtools/datasets/}}
Besides, we also compare these methods on the \emph{CIFAR-10} database for image classification.\footnote{\url{https://www.cs.toronto.edu/~kriz/cifar.html}}

\subsubsection{Classification Results on UCI database}
\label{sec:classUCI}
Ten small data sets from the UCI database are used to evaluate our DANK model embedded in SVM.
Here we describe experimental settings and the compared algorithms as follows.

\noindent{\bf Experimental Settings:} Table~\ref{UCIres} lists a brief description of these ten data sets including the number of training data $n$ and the feature dimension $d$.
After normalizing the data to $[0, 1]^d$ by a min-max scaler, we randomly pick half of the data for training and the rest for test except for \emph{monks1}, \emph{monks2}, and \emph{monks3}.
In these three data sets, both training and test data have been provided.
The Gaussian kernel $k(\bm x_i, \bm x_j) = \exp(-{\| \bm x_i - \bm x_j \|_2^2}/{2\sigma^2})$ is chosen as the initial kernel in our model.
The kernel width $\sigma$ and the balance parameter $C$ are tuned by 5-fold cross validation on a grid of points, i.e., $\sigma=[2^{-5}, 2^{-4},\dots,2^5]$ and $C = [2^{-5}, 2^{-4},\dots,2^5]$.
To avoid additional cross validation, we manually set the penalty parameter $\tau$ to 0.01.
The regularization parameter $\eta$ is fixed to $\| \bm \alpha \|_2^2$ obtained by SVM.
The experiments are conducted 10 times on these ten data sets.

\noindent{\bf Compared Methods:} We include the following kernel learning based algorithms:
\begin{itemize}
	\item BMKL \citep{Gonen2012Bayesian}: A multiple kernel learning algorithm uses Bayesian approach to ensemble the Gaussian kernels with ten different kernel widths and the polynomial kernels with three different degrees.
	\item LogDet \citep{Jain2012Metric}: A nonparametric kernel learning approach aims to learn a PSD matrix $\bm W$ in a learned kernel $\phi(\bm x)^{\!\top} \bm W  \phi(\bm x')$ with the LogDet divergence.
	\item RF \citep{AmanNIPS2016}: A kernel alignment based learning framework creates randomized features, and then solves a simple optimization problem to select a subset.
	Finally, the kernel is learned from the optimized features by target alignment.
	\item KNPL \citep{Liu2018AAAIme}: This nonparametric kernel learning framework is given by our conference version, which shares the initial ideas about learning in a data-adaptive scheme via a pair-wise way. But this work does not consider the bounded constraint and the low-rank structure on $\bm F$, and utilizes an alternating iterative algorithm to solve the corresponding semi-definite programming.
	\item MKL-uniform: It is a multiple kernel learning algorithm with uniform weights, and serves as a baseline. It uses 11 equal-weighted Gaussian kernels with the kernel width $\sigma=[2^{-5}, 2^{-4},\dots,2^5]$, respectively. This setting avoids tuning the kernel width $\sigma$ but the balance parameter $C$ is still tuned by 5-fold cross validation.
	\item SVM-CV: The SVM classifier with cross validation serves as a baseline.
\end{itemize}

\begin{table*}[t]
	\centering
	\scriptsize
	\caption{\small Comparison results with two baselines in terms of classification accuracy (mean$\pm$std. deviation \%) on ten UCI data sets. The best performance is highlighted in bold. The classification accuracy on the training data is presented by italic, and does not participate in ranking. Notation ``$\bullet$" indicates that our DANK method is significantly better than other baseline methods via paired t-test at the 5\% significance level.}
	\begin{small}
		\begin{tabular}{ccccccccccccccccccccc}
			\toprule[1.5pt]
			\multirow{2}{*}{Data set} &\multirow{2}{*}{($d$, $n$)} &\multicolumn{2}{c}{MKL-uniform}  &\multicolumn{2}{c}{SVM-CV} &\multicolumn{2}{c}{DANK}\cr
			\cmidrule(lr){3-4} \cmidrule(lr){5-6} \cmidrule(lr){7-8}
			{}&{}&Training&Test
			&Training&Test  &Training&Test\cr
			\midrule[1pt]
			diabetic &(19, 1151)	&	\emph{83.5$\pm$14.2}	&	61.9$\pm$8.5	&	\emph{80.7$\pm$3.9}	&	73.0$\pm$1.7 &	\emph{87.0$\pm$1.9}&	{81.2$\pm$1.4}$\bullet$	\\
			\hline
			heart &(13, 270)	&	\emph{95.8$\pm$4.3}	&	82.1$\pm$2.9 		&	\emph{88.9$\pm$3.0}	&	81.9$\pm$2.4 	&	\emph{94.3$\pm$1.7}	&	{\bf87.9}$\pm$2.9$\bullet$	\\
			\hline
			monks1 &(6, 124)	&	\emph{99.3$\pm$1.3}	&	79.0$\pm$2.8 	&	\emph{90.3$\pm$0.0}	&	81.4$\pm$0.0 	&	\emph{100.0$\pm$0.0}	&	{\bf 83.6}$\pm$1.5$\bullet$	\\
			\hline
			monks2 &(6, 169)	&	\emph{100.0$\pm$0.0}	&	84.4$\pm$0.1	&	\emph{100.0$\pm$0.0}	&	{85.8}$\pm$1.4&	\emph{100.0$\pm$0.0}	&{\bf 86.7}$\pm$0.9\\
			\hline
			monks3 &(6, 122)	&	\emph{99.6$\pm$1.3}	&	90.8$\pm$0.6 	&	\emph{96.2$\pm$1.5}	&	{\bf 93.0}$\pm$1.2		&\emph{97.2$\pm$1.8}		&{\bf 93.0}$\pm$0.9	\\
			\hline
			sonar &(60, 208)	& \emph{100.0$\pm$0.0}	&	80.8$\pm$3.7 	&	\emph{99.9$\pm$0.3}	&	85.3$\pm$3.1  &	\emph{100.0$\pm$0.0}	&	{\bf87.0}$\pm$2.7$\bullet$	\\ \hline
			spect &(21, 80)	&\emph{93.8$\pm$0.0}	&	{\bf 79.8}$\pm$0.3 	&	\emph{87.0$\pm$3.7}	&	73.1$\pm$3.2	 &	\emph{93.4$\pm$4.1}	&	{78.9$\pm$3.2}	\\
			\hline
			glass	 &(9, 214) 	&	\emph{93.4$\pm$6.0}	& 53.0$\pm$7.0 	&	\emph{77.1$\pm$6.9}	&	69.8$\pm$2.0	 &	\emph{89.7$\pm$6.1}	&	{\bf 74.5}$\pm$1.3$\bullet$	\\
			\hline
			fertility	 &(9, 100) 	&	\emph{99.0$\pm$1.4}	&85.6$\pm$3.0	&	\emph{94.4$\pm$5.3}	&	85.2$\pm$1.7	 &	\emph{97.3$\pm$3.3}	&	{\bf 87.6}$\pm$2.3$\bullet$	\\
			\hline
			wine	 &(13, 178) 	&	\emph{100.0$\pm$0.0}	&96.0$\pm$3.8	&	\emph{99.5$\pm$1.0}	&	94.7$\pm$1.5	 &	\emph{99.5$\pm$1.0}	&	{\bf96.4}$\pm$2.0	\\
			\bottomrule[1.5pt]
		\end{tabular}
	\end{small}
	\label{SOCMKL}
\end{table*}

\begin{table*}[t]
	\centering
	\scriptsize
	\caption{\small Comparison results of several representative kernel learning based algorithms in terms of test accuracy on ten UCI dataaset. The best performance is highlighted in bold. Notation ``$\bullet$" indicates that the data-adaptive based algorithm (KNPL or DANK) is significantly better than other representative kernel learning methods via paired t-test at the 5\% significance level.}
	\begin{small}
		\begin{tabular}{cccccccccccccccccccc}
			\toprule[2pt]
			{Data set}  &{LogDet} &{BMKL} &{RF} &{KNPL} &{DANK}\cr
			\midrule[1pt]
			diabetic	&	78.7$\pm$1.8	&	{74.9}$\pm$0.4	&	72.3$\pm$0.8		&	{\bf81.9}$\pm$1.7$\bullet$ &	{81.2$\pm$1.4}$\bullet$	\\
			\hline
			heart 	&	80.6$\pm$3.5	&	85.6$\pm$0.8	&	79.1$\pm$2.4 	 &	87.4$\pm$3.9$\bullet$	&	{\bf87.9}$\pm$2.9$\bullet$	\\
			\hline
			monks1	&	{\bf 86.6}$\pm$0.2	&	78.9$\pm$2.5	&	{84.4}$\pm$0.9 &		83.3$\pm$3.3		&	83.6$\pm$1.5	\\
			\hline
			monks2	&	85.9$\pm$1.2	&	82.1$\pm$1.3	&	{73.6$\pm$1.1}		&	83.3$\pm$1.6	&{\bf 86.7}$\pm$0.9\\
			\hline
			monks3 	&	{\bf 94.0}$\pm$1.3	&	{\bf94.0}$\pm$1.0	&	93.7$\pm$0.6 	&	88.7$\pm$1.2 &	93.0$\pm$0.9	\\
			\hline
			sonar 	&	84.1$\pm$2.2	&	84.8$\pm$0.6	&	80.5$\pm$3.1 	 &	85.8$\pm$2.8 &	{\bf87.0}$\pm$2.7$\bullet$	\\ \hline
			spect	&	79.6$\pm$3.7	&	{78.8}$\pm$0.8	&	76.0$\pm$2.7 		&	{\bf 79.7}$\pm$4.8 	&	{78.9$\pm$3.2}	\\
			\hline
			glass	&	71.2$\pm$1.0	&	68.2$\pm$4.6	& { 68.2}$\pm$2.2 		&	72.4$\pm$2.3 &	{\bf 74.5}$\pm$1.3$\bullet$	\\
			\hline
			fertility &	86.2$\pm$1.1	&	84.4$\pm$1.6	&84.4$\pm$4.3	&	85.6$\pm$3.8 	&	{\bf 87.6}$\pm$2.3$\bullet$	\\
			\hline
			wine	 &	96.1$\pm$1.8	&	95.0$\pm$2.8	&95.1$\pm$1.1	&	 96.1$\pm$2.0 &	{\bf96.4}$\pm$2.0	\\
			\bottomrule[2pt]
		\end{tabular}
	\end{small}
	\label{UCIres}
\end{table*}

{\bf Experimental Results:}
	We first evaluate the proposed DANK model with two baselines: MKL-uniform and SVM-CV in terms of classification accuracy on the training and test data in Table~\ref{SOCMKL}, and then compare DANK with other representative kernel learning based algorithms in Table~\ref{UCIres}.
	
	In Table~\ref{SOCMKL}, MKL-uniform sometimes performs the best on the training data, but fails to generalize on the test data in most cases. In general, it is inferior to SVM-CV and our DANK model in terms of the test accuracy. Directly enlarging the solving space without any constraint would increase model flexibility but is easy to be over-fitting.
	Compared with the baseline SVM-CV, the proposed DANK model achieves a good trade-off between model flexibility and complexity. Training accuracy on \emph{diabetic}, \emph{heart}, \emph{monks1}, \emph{spect}, and \emph{glass} indicates the effectiveness of our data adaptive scheme on increasing the model flexibility.
	Accordingly, this strategy is helpful for our model to achieve noticeable improvements on the test data.
	On the \emph{monks2}, \emph{sonar}, and \emph{wine} data sets, SVM-CV has already obtained nearly 100\% accuracy on the training data, which indicates that the model flexibility is sufficient.
	In this case, it is difficult for our DANK method to achieve a huge improvement on these data sets, and accordingly the performance margins are about 0\%$\sim$2\%.
Table~\ref{UCIres} reports the test accuracy of typical kernel learning based algorithms.
We also apply the paired t-test at the 5\% significance level to investigate whether the data-adaptive approaches (KNPL and DANK) are significantly better than other methods.
It can be found that, compared with representative kernel learning based algorithms including LogDet, BMKL, and RF, the proposed DANK model yields favorable performance.

In general, the improvements on the classification accuracy demonstrate the effectiveness of our data adaptive scheme, and accordingly our model has good adaptivity to the training and test data.

\subsubsection{Results on large-scale data sets}

\begin{table}
	\centering
	\small
	\caption{Large-sample data set statistics and parameter settings.}
	\label{tablarge}
	\begin{threeparttable}
		\begin{tabular}{cccccccccccccc}
			\toprule[1.5pt]
			data sets &d & \#training & \#test & $C$ & $1/2\sigma^2$ &\#clusters\\
			\midrule[1pt]
			\emph{EEG}  &14 &7,490 &7,490 &32 & 100 &5
			\\
			\hline
			\emph{ijcnn1}  &22 &49,990 &91,701 &32 & 2 &50
			\\
			\hline
			\emph{covtype}  &54 &464,810 &116,202 &32 & 32 &200
			\\
			\bottomrule[1.5pt]
		\end{tabular}
	\end{threeparttable}
	\vspace{-0.35cm}
\end{table}

\begin{table*}[t]
	\centering
	\fontsize{9}{8}\selectfont
	\begin{threeparttable}
		\caption{\small Comparison of test accuracy and training time of all the compared algorithms on several large data sets. The rank of $\bm F$ in our DANK model is also given in bold.}
		\label{tablarges}
		\begin{tabular}{cccccccccccccccccccc}
			\toprule[1.5pt]
			\multicolumn{2}{c}{Method} &SVM-SMO  &
			\multicolumn{2}{c}{LogDet} &\multicolumn{2}{c}{BMKL} &RF &\multicolumn{2}{c}{DANK}\cr
			\cmidrule(lr){1-2} \cmidrule(lr){3-3} \cmidrule(lr){4-5} \cmidrule(lr){6-7} \cmidrule(lr){8-8} \cmidrule(lr){9-10}
			Data set&Setting &exact &exact&scalable &exact&scalable &exact &exact(rank.) &scalable(rank.) \cr
			\midrule[1pt]
			\multirow{3}{0.5cm}{\emph{EEG}} &acc.(\%) &95.9 &95.9 &95.2  &94.5 &94.1 &80.1 &96.7({\bf 36}) &96.3({\bf 142}) \\
			\cmidrule(lr){2-2} \cmidrule(lr){3-3} \cmidrule(lr){4-5} \cmidrule(lr){6-7} \cmidrule(lr){8-8} \cmidrule(lr){9-10}
			&time(sec.)  &8.6 &211.6 &24.6 &1426.3 &124.5 &2.4 &473.6 &39.8 \cr
			\midrule[1pt]
			\multirow{3}{0.5cm}{\emph{ijcnn1}} &acc.(\%) &96.5 &97.9 &97.4  &98.5 &97.7 &93.0 &98.9({\bf 342}) &98.4({\bf 1788}) \\
			\cmidrule(lr){2-2} \cmidrule(lr){3-3} \cmidrule(lr){4-5} \cmidrule(lr){6-7} \cmidrule(lr){8-8} \cmidrule(lr){9-10}
			&time(sec.)  &112.4 &8472.3 &78.1 &109967 &1916.9 &25.0 &28548 &571.7 \cr
			\midrule[1pt]
			\multirow{3}{1cm}{\emph{covtype}} &acc.(\%) &96.1 &$\times$\tnote{1} &96.4 &$\times$ &91.2 &79.1 &$\times$ &97.1  \\
			\cmidrule(lr){2-2} \cmidrule(lr){3-3} \cmidrule(lr){4-5} \cmidrule(lr){6-7} \cmidrule(lr){8-8} \cmidrule(lr){9-10}
			&time(sec.) &3972.5  &$\times$ &5021.4 &$\times$ &94632 &364.5  &$\times$ &7534.2  \cr
			\bottomrule[1.5pt]
		\end{tabular}
		\begin{tablenotes}
			\footnotesize
			\item[1] These methods attempt to directly solve the optimization problem on \emph{covtype} but fail due to the memory limit.
		\end{tablenotes}
	\end{threeparttable}
\end{table*}

To validate our decomposition scheme on large scale situations, we consider three large data sets including \emph{EEG}, \emph{ijcnn1}, and \emph{covtype} for comparisons.
Table~\ref{tablarge} reports the data set statistics (i.e., the feature dimension $d$, the number of training samples, and the number of test data) and parameter settings including the balance parameter $C$, the kernel width $\sigma$, and the number of clusters.
Table~\ref{tablarges} presents the test accuracy and training time of various compared algorithms including  LogDet, BMKL, our DANK method, SVM-SMO \citep{Platt1999Fast} (the cache is set to 5000) and RF conducted in the following two settings.

In the first setting (``exact"), we attempt to directly test these algorithms over the entire training data.
Experimental results indicate that, without the decomposition-based scalable approach, our DANK method achieves the best test accuracy with 96.7\% sand 98.9\% on \emph{EEG} and \emph{ijcnn1}, respectively.
However, under this setting, LogDet, BMKL, and our method fail to deal with an extreme large data set \emph{covtype} due to the memory limit except SVM-SMO and RF.

In the second setting (``scalable"), we incorporate our kernel approximation scheme into LogDet, BMKL, and DANK evaluated on the three large data sets.
Experimental results show that, by such decomposition-based scalable approach, we can speed up the above three kernel learning methods.
For example, when compared with the direct solution of the optimization problem in the ``exact" setting, LogDet, BMKL, and DANK equipped with kernel approximation speed up about 100x, 50x, and 50x on \emph{ijcnn1}, respectively.
More importantly, on these three data sets, our DANK method using the approximation scheme still performs better than SVM-SMO on the test accuracy, which demonstrates the effectiveness of the proposed non-parametric kernel learning framework.

Results in above two settings show that, our DANK method achieves promising test accuracy no matter whether the kernel approximation scheme is incorporated or not.
Moreover, such approximation scheme makes BMKL, LogDet, and our DANK method feasible to large data sets with huge speedup in terms of computational efficiency.

\subsubsection{Results on CIFAR-10 data set}

\begin{figure}
	\centering
	\pgfplotstableread[row sep=\\,col sep=&]{
		interval & carT  \\
		DANK     & 92.68  \\
		KNPL   & 91.72  \\
		VGG16   & 90.87   \\
		SVM-CV     & 90.35  \\
		RF  & 87.56 \\
		BMKL & 90.55 \\
		LogDet & 91.47 \\
	}\mydata
	
	\begin{tikzpicture}[scale=0.65]
	\begin{axis}[
	ybar,
	ymajorgrids=true,
	bar width=.5cm,
	width=.9\textwidth,
	height=.5\textwidth,
	legend style={at={(0.5,1)},
		anchor=north,legend columns=-1},
	symbolic x coords={DANK, KNPL, VGG16, SVM-CV, RF, BMKL, LogDet},
	xtick=data,
	nodes near coords,
	nodes near coords align={vertical},
	ymin=80,ymax=95,
	ylabel={\Large Accuracy(\%)},
	]	
	\addplot+[error bars/.cd,
	y dir=both,y explicit]
	coordinates {
		(DANK,92.68) +- (0.0, 0.41)
		(KNPL,91.72) +- (0.0, 0.52)
		(VGG16,90.87) +- (0.0, 0.54)
		(SVM-CV,90.35) +- (0.0, 0.18)
		(RF,87.56) +- (0.0, 0.92)
		(BMKL,90.55) +- (0.0, 0.58)
		(LogDet,91.47) +- (0.0, 0.44)};
	\legend{\Large{classification accuracy}}
	\end{axis}
	\end{tikzpicture}
	\caption{Performance of the compared algorithms on \emph{CIFAR-10} data set.}\label{cifar10}
	\vspace{-0.1cm}
\end{figure}

In this section, we test our model on a representative data set \emph{CIFAR-10} \citep{Krizhevsky2009Learning} for natural image classification task.
This data set contains 60,000 color images with the size of $32 \times 32 \times 3$  in 10 categories, of which 50,000 images are used for training and the rest are for testing.
In our experiment, each color image is represented by the feature extracted from a convolutional neural network, i.e., VGG16 with batch normalization \citep{ioffe2015batch} pre-trained on ImageNet \citep{Deng2009ImageNet}. Then we fine-tune this pre-trained VGG16 model on the CIFAR10 data set with 240 epochs and a min-batch size of 64.
The learning rate starts from 0.1 and then is divided by 10 at the 120-th, 160-th, and 200-th epoch.
After that, for each image, a 4096 dimensional feature vector is obtained according to the output of the first fully-connected layer in this fine-tuned neural network.

Figure~\ref{cifar10} shows the test accuracy (mean$\pm$std. deviation \%) of the compared algorithms averaged in ten trials.
The original VGG16 model with the softmax classifier achieves 90.87$\pm$0.54\% on the test accuracy.
Using the extracted 4096-D feature vectors, kernel learning based algorithms equipped with the initial Gaussian kernel are tuned by 5-fold cross validation on a grid of points, i.e., $\sigma=[0.001,0.01,0.1,1]$ and $C = [1,10,20,30,40,50,80,100]$.
In terms of classification performance, SVM-CV, BMKL, LogDet, and KNPL obtain 90.35$\pm$0.18\%, 90.55$\pm$0.58\%, 91.47$\pm$0.44\%, and 91.72$\pm$0.52\% accuracy on the test data, respectively.
Comparably, our DANK model achieves promising classification accuracy with 92.68$\pm$0.41\%.
More importantly, it outperforms SVM-CV with an accuracy margin of 2.33\%, and is \emph{statistically significant} better than the other methods via paired t-test at the 5\% significance level.
The improvement over SVM-CV on the test accuracy demonstrates that our DANK method equipped with the introduced kernel adjustment strategy is able to enhance the model flexibility, and thus achieves good performance.

\subsection{Analysis and Validation for Theoretical Results} 
In this subsection, we experimentally validate the effectiveness of the used Nesterov's smooth optimization method, the rationality of dropping out the low-rank regularizer in the large sample case, the robustness of $\tau$ via parameter sensitivity analysis, and the tight bound of our theoretical results.
\subsubsection{Convergence experiments}
\begin{figure}
	\begin{center}
		\subfigure[$H(\bm \alpha, \bm F)$]{\label{subobjf}
			\includegraphics[width=0.35\textwidth]{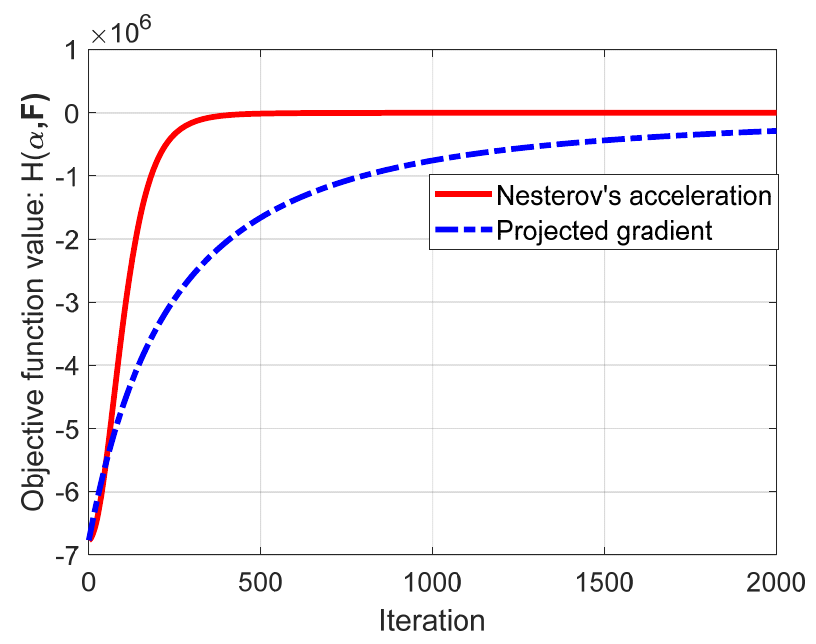}}
		\hspace{1cm}
		\subfigure[$\| \bm \alpha^{(t)} - \bm \alpha^{(t-1)} \|_2$]{\label{subala}
			\includegraphics[width=0.36\textwidth]{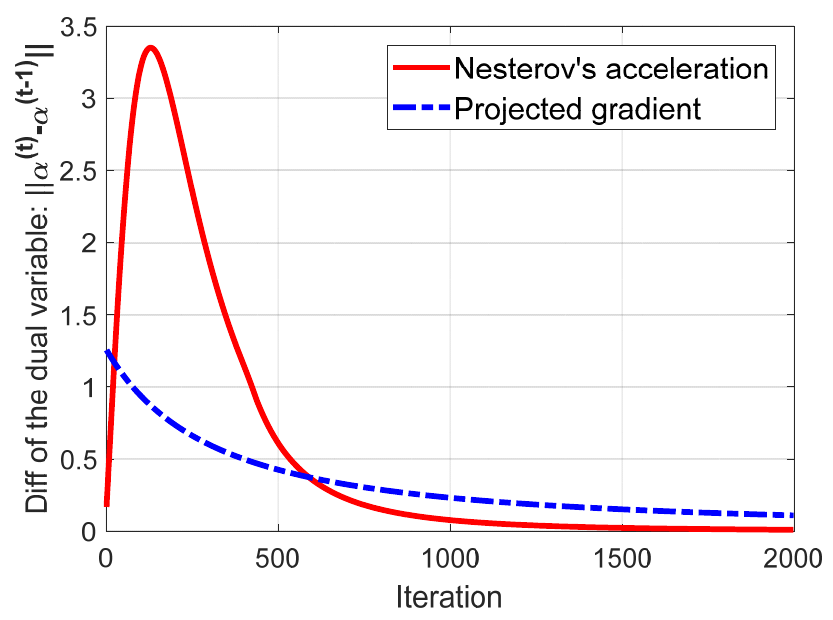}}
		\caption{Comparison between projected gradient method and our Nesterov's acceleration on the \emph{heart} data set.}
		\label{conv}
	\end{center}\vspace{-0.8cm}
\end{figure}

To investigate the effectiveness of the used Nesterov's smooth optimization method, we conduct a convergence experiment on the \emph{heart} data set.

In Figure~\ref{subobjf}, we plot the objective function value $H(\bm \alpha, \bm F)$ versus iteration by the standard projected gradient method (in blue dashed line) and its Nesterov's acceleration (in red solid line), respectively.
One can see that the developed first-order Nesterov's smooth optimization method converges faster than the projected gradient method, so the feasibility of employing Nesterov's acceleration for solving problem~\eqref{mainrank} is verified.
Besides, to further illustrate the convergence of $\{ \bm \alpha^{(t)} \}_{t=0}^{\infty}$, we plot $\| \bm \alpha^{(t)} - \bm \alpha^{(t-1)} \|_2$ versus iteration in Figure~\ref{subala}. We find that the sequence $\{ \bm \alpha^{(t)} \}_{t=0}^{\infty}$ yielded by the Nesterov's acceleration algorithm significantly decays in the first 500 iterations, which leads to a fast convergence to the optimal solution.
Hence, compared with projected gradient method, the Nesterov's smooth optimization method is able to efficiently solve the targeted convex optimization problem in this paper.

\subsubsection{Validation for the low-rank constraint}
\label{sec:explowrank}
\begin{table*}[t]
	\centering
	\scriptsize
	\caption{\small Influence of the low-rank constraint with different values of $\tau$ on test classification accuracy. Notation ``$\bullet$" indicates that test accuracy by this setting $\tau$ is statistical different from that of the current setting $\tau=0.01$ via paired t-test at the 5\% significance level. The rank of $\bm F$ in our DANK model is also given in bold.}
	\begin{small}
		\begin{tabular}{cccccccccccccccccccc}
			\toprule[2pt]
			{$\tau$}  &{0} &{0.001} &{0.01} &{0.1} &{1}\cr
			\midrule[1pt]
			diabetic	&	78.1$\pm$1.2 ({\bf 340})$\bullet$	&	81.1$\pm$1.2 ({\bf 10})	&	{81.2$\pm$1.4} ({\bf 5})		&	{81.2}$\pm$1.2 ({\bf 2}) &	{81.1}$\pm$1.2 ({\bf 2})	\\
			\hline
			heart 	&	84.9$\pm$2.2 ({\bf 109})$\bullet$	&	87.3$\pm$2.2 ({\bf 5})	&	87.9$\pm$2.9 ({\bf 3})	 &	86.5$\pm$2.5 ({\bf 2})$\bullet$	&	85.4$\pm$2.8 ({\bf 1})$\bullet$	\\
			\hline
			monks1	&	{82.1}$\pm$1.3 ({\bf 68})$\bullet$	&	82.5$\pm$1.4 ({\bf 15})$\bullet$	&	83.6$\pm$1.5 ({\bf 6}) &		83.2$\pm$1.4 ({\bf 2})		&	81.7$\pm$1.3 ({\bf 1})$\bullet$	\\
			\hline
			monks2	&	85.5$\pm$0.8 ({\bf 68})$\bullet$	&	86.2$\pm$0.7 ({\bf 13})$\bullet$	&	{86.7}$\pm$0.9 ({\bf 4})	&	86.4$\pm$0.8 ({\bf 2})	&{86.0}$\pm$0.9 ({\bf 2})$\bullet$ \\
			\hline
			monks3 	&	90.2$\pm$1.3 ({\bf 61})$\bullet$	&	92.8$\pm$1.0 ({\bf 8})	&	93.0$\pm$0.9 ({\bf 4})	&	93.3$\pm$1.1 ({\bf 2}) &	92.8$\pm$1.1 ({\bf 1})	\\
			\hline
			sonar 	&	84.0$\pm$2.4 ({\bf 81})$\bullet$	&	85.2$\pm$2.4 ({\bf 32})$\bullet$	&	87.0$\pm$2.7 ({\bf 17})	 &	87.4$\pm$2.6 ({\bf 7}) &	86.8$\pm$2.3 ({\bf 3})$\bullet$	\\ \hline
			spect	&	78.2$\pm$2.6 ({\bf 15}) 	&	{78.4}$\pm$2.8 ({\bf 8}) 	&	78.9$\pm$3.2 ({\bf 6}) 		&	{78.4}$\pm$2.8 ({\bf 3})  	&	{76.4}$\pm$2.5 ({\bf 2})$\bullet$ 	\\
			\hline
			glass	&	72.3$\pm$1.1 ({\bf 34})$\bullet$	&	73.7$\pm$1.4 ({\bf 14})$\bullet$	& 74.5$\pm$1.3 ({\bf 8})		&	73.3$\pm$1.2 ({\bf 3})$\bullet$ &	 71.4$\pm$1.5 ({\bf 2})$\bullet$	\\
			\hline
			fertility &	86.8$\pm$2.1 ({\bf 29})$\bullet$	&	87.1$\pm$2.2 ({\bf 11})	&87.6$\pm$2.3 ({\bf 5})	&	87.3$\pm$2.1 ({\bf 3}) 	& 87.2$\pm$2.2 ({\bf 2})	\\
			\hline
			wine	 &	96.1$\pm$2.0 ({\bf 5})	&	96.2$\pm$1.9 ({\bf 3})	&96.4$\pm$2.0 ({\bf 2})	&	 96.3$\pm$2.0 ({\bf 2}) &	96.3$\pm$2.0 ({\bf 2})	\\
			\bottomrule[2pt]
		\end{tabular}
	\end{small}
	\label{UCIrank}
\end{table*}

Due to the inseparable property of the low-rank constraint, we omit the low-rank regularizer on $\bm F$ for efficient optimization in the large sample case.
Here we experimentally specialize in the influence of $\| \bm F \|_*$ on the test classification accuracy in both small and large data sets.

Regarding to the small sample case, Table~\ref{UCIrank} reports the classification accuracy and the rank of $\bm F$ under different values of $\tau$ on these ten data sets appeared in Section \ref{sec:classUCI}.
Compared to our model with $\tau = 0.01$, the setting without the low-rank regularizer (i.e., $\tau = 0$) loses about 1\%$\sim$3\% accuracy on most data sets with statistical difference except for the \emph{spect} and \emph{wine} data sets.
On these two data sets, we find that, although the low-rank regularizer is not considered, the learned $\bm F$ still exhibits the low-rank structure, which effectively controls the model flexibility and complexity. Accordingly, our model without the low-rank constraint performs well on the two data sets.
However, in the remaining data sets, the rank of $\bm F$ significantly increases if we drop out the low-rank regularizer.
In this case, the solving space of our model could be extreme large, which would lead to over-fitting.
Besides, the learned $\bm F$ might be sophisticated, therefore, it is not easily extended to $\bm F'$ for test data by the simple nearest neighbor scheme.
Hence, results in Table~\ref{UCIrank} indicates that the low-rank constraint is important in small-sample data sets.

In terms of large sample case, according to Table~\ref{tablarges}, our DANK model with the ``exact" solution achieves 96.7\% and 98.9\% accuracy on the \emph{EEG} and \emph{ijcnn1} data set, respectively. 
In contrast, after omitting the low-rank regularizer $\| \bm F \|_*$, the test accuracy of our model in the ``scalable" setting decreases to 96.3\% and 98.4\%, respectively.
More importantly, we find that, without the low-rank constraint, the rank of $\bm F$ increases from 36 (in the ``exact" setting) to 142 (in the ``scalable" setting) on the \emph{EEG} data set with 7,490 training samples.
This tendency also exhibits on the \emph{ijcnn1} data set with 49,900 training samples.
The rank of $\bm F$ on this data set increases from $\text{rank}(\bm F) = 342$ to $\text{rank}(\bm F) = 1788$.
So the above results indicate that dropping the low-rank constraint leads to a slight decrease on the test accuracy; while the rank of $\bm F$ is indeed raising but is still much smaller than $n$.
In the next, we explain the reason why our model still obtains a relative low-rank matrix $\bm F$ in the large sample case.
On one hand, the regularization parameter $\eta$ is chosen as $\eta:= \| \bm \alpha \|_2^2 \in \mathcal{O}(n)$ in our experiment.
This would be an extreme large value in large sample data sets, resulting in a strong regularization term $\eta \| \bm F - \bm 1 \bm 1^{\!\top} \|_{\F} $.
Therefore, $\bm F$ varies around the all-one matrix in a small bounded region and thus shows a low-rank effect to some extent.
On the other hand, the used Gaussian kernel inherits the rapid decaying spectra \citep{smola2000sparse}, e.g., the exponential decay $\lambda_i \propto ne^{-ai}$ with $a > 0$ as illustrated by \cite{bach2013sharp}. 
As discussed in Section~\ref{sec:ana}, when the Gaussian kernel is chosen as the initial one, the used centering regularizer is helpful to obtain a relative low-rank kernel matrix $\widetilde{\bm K}$.
Based on the above analysis, in large sample case, using the centering regularizer $\eta \| \bm F - \bm 1 \bm 1^{\!\top} \|_{\F} $ could be an alternative choice for the low-rank property if we have to drop $\| \bm F \|_*$.

From above observations and analyses, we conclude that the low-rank constraint in our model is important for small-sample cases.
In large-sample data sets, due to the separability requirement, our DANK model has to drop the low-rank constraint, and thus achieving a slight decrease on the final classification accuracy. 
Nevertheless, the employed centering regularizer restricts $\bm F$ to a small bounded region around the all-one matrix, which would be an alternative choice to seek for a low-rank matrix $\bar{\bm F}$.

\subsubsection{Parameter sensitivity analysis}
\label{sec:paramana}
	The above subsection demonstrates the importance of the low-rank regularizer in our DNAK model on small sample case, so here we need to investigate the parameter sensitivity of $\tau$ to the test accuracy on these data sets.
	
	Table~\ref{UCIrank} reports the classification accuracy with $\tau$ in $\{ 0,0.001,0.01,0.1,1 \}$.
	We find that, when $\tau$ is changed from $0$ to $0.01$, the rank of $\bm F$ rapidly decreases, which implies that the model flexibility is effectively controlled. The test accuracy yielded by our DANK model is accordingly improved under this range.
	When $\tau$ further increases to $\tau=0.1$, there is no \emph{statistical significance} found between $\tau=0.01$ and $\tau=0.1$ except on the \emph{heart} and \emph{glass} data sets in terms of classification accuracy.
	Furthermore, if $\tau$ increases to $1$, $\bm F$ exhibits an extremely low-ranking structure. In this case, our model with $\tau=1$ is not flexible enough to fit the data, and thus is inferior to the setting of $\tau=0.01$ on most data sets.
	
	Based on the above observations, our DANK model is robust to the variations of $\tau$ ranging from $0.01$ to $0.1$, so it can be easily tuned and we suggest $\tau=0.01$ for practical use.

\subsubsection{Validation for our derived bounds}
\label{sec:bound}
\begin{figure}
	\centering
		\subfigure[\emph{ijcnn1}]{
		\includegraphics[width=0.31\textwidth]{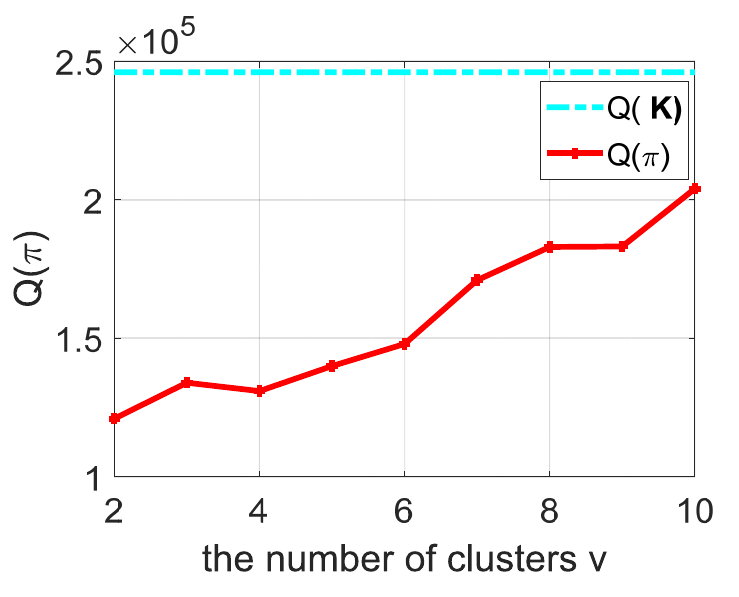}}
	\subfigure[\emph{covtype}]{
		\includegraphics[width=0.31\textwidth]{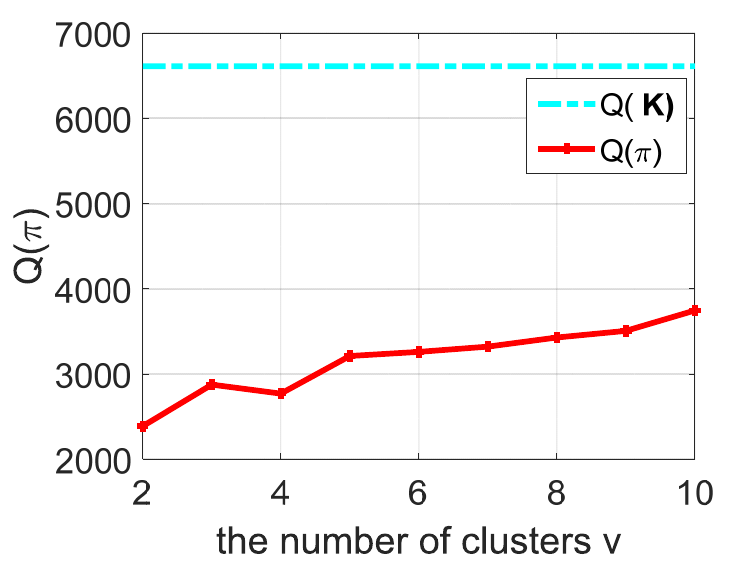}}
	\subfigure[non-support vectors]{\label{gradbound}
		\includegraphics[width=0.305\textwidth]{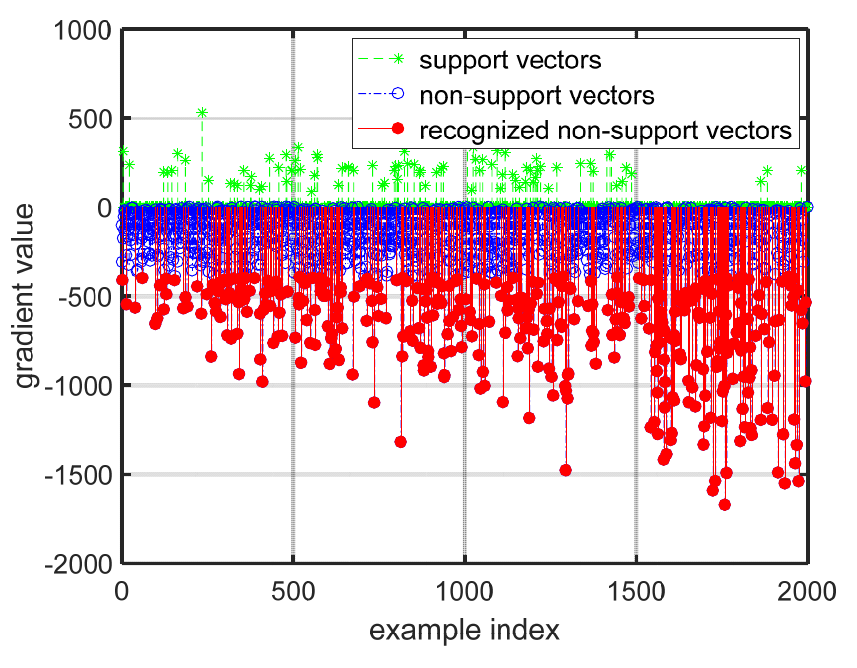}}
	\caption{Experimental validation of the derived bounds in Theorem~\ref{theorappf} and~\ref{theorsv}. }\label{thm4bound}
\end{figure}
Here we show that the derived bounds in Theorem~\ref{theorappf} and~\ref{theorsv} are tight in practice.

We firstly study the relationship between $Q(\pi)$ and $v$ in our model as shown in Figure~\ref{thm4bound}.
In our experiments, we randomly select 2,000 samples from the \emph{ijcnn1} and \emph{covtype} data set, and then group them into $v=2,3,\cdots,10$ clusters, respectively.
It can be observed that $Q(\pi)$ almost increases linearly with $v$ on these two data sets, which shows consistency with theoretical results in \citep{giraud2019partial} to some extent.
Further, if we consider more clusters, such as $v=100,200,\cdots$ (in practice, this situation is rare in clustering algorithms with only 2,000 samples but over 100 clusters), $Q(\pi)$ would increase slowly to approach to the upper bound $ Q(\bm K) := \sum_{i,j=1}^n |K_{ij}|$.

In Figure~\ref{gradbound}, we validate that the derived condition in Theorem~\ref{theorsv} is useful for screening non-support vectors.
We follow with the above experimental setting with $v=10$, and obtain the approximation solution $(\bar{\bm \alpha}, \bar{\bm F})$. 
Then we compute the gradient $\Big(\! \nabla_{\bm \alpha} \bar{H}(\bar{\bm \alpha}, \bar{\bm F}) \!\Big)_{\!i}$ associated with each sample and plot the support vectors (in green) and non-support vectors (in blue) as shown in Figure~\ref{gradbound}.
Under the condition of the condition~\eqref{gradsvc} in Theorem~\ref{theorsv}, the non-support vectors that satisfy $\Big(\! \nabla_{\bm \alpha} \bar{H}(\bar{\bm \alpha}, \bar{\bm F}) \!\Big)_{\!i} \leq (B \!+\! B_2)C \left(\| \bar{\bm K}_i \|_1 \!+\! \kappa \right) = 442.3$ can be directly recognized as non-support vectors (marked in red) without solving the optimization problem.
We find that in the total of \#SVs=214 (the number of support vectors) and \#non-SVs=1786 (the number of non-support vectors), there are 723 non-support vectors recognized by our Theorem~\ref{theorsv}.
That means, over 40\% non-support vectors can be picked out, which demonstrates the effectiveness of our derived bound/condition in Theorem~\ref{theorsv}.

\subsection{Regression}
This section focuses on the proposed DANK model embedded in SVR for regression tasks.
We firstly conduct the experiments on several synthetic data sets, to examine the performance of our method on recovering 1-D and 2-D test functions.
Then we evaluate our model with representative regression algorithms on real-world UCI data sets.\footnote{The data sets are available at \url{http://www.csie.ntu.edu.tw/~cjlin/libsvmtools/datasets/binary.html} }
The used evaluation metric here is relative mean square error (RMSE) between the learned regression function $\hat{g}(\bm x)$ and the target label $\bm y$ over $n$ data points
\begin{equation*}
\mbox{RMSE} = \frac{\sum_{i=1}^n\big(\hat{g}(\bm x_i) - y_i \big)^2}{\sum_{i=1}^{n} \big( y_i - \mathbb{E}(\bm y) \big)^2} \,.
\end{equation*}

\subsubsection{Synthetic Data}
Here we test the approximation performance of our method on 1-D and 2-D test functions compared with a baseline, the SVR with Gaussian kernel.
The representative 1-D step function is defined by
\begin{equation*}
g(s,w,a,x) = \bigg( \frac{\tanh \big( \frac{ax}{w} -a \lfloor \frac{x}{w}\rfloor - \frac{a}{2} \big) }{2 \tanh \big( \frac{a}{2} \big)} + \frac{1}{2} + \left\lfloor \frac{x}{w} \right\rfloor \bigg)s\,,
\end{equation*}
where $s$ is the step hight, $w$ is the period, and $a$ controls the smoothness of the function $g$.
In our experiment, $s,w$ and $a$ are set to 3, 2 and 0.05, respectively.
We plot the step function on $[-5,5]$ as shown in Figure~\ref{figstep}.
One can see that the approximation function generated by SVR-CV (blue dashed line) yields a larger deviation than that of our DANK model (red solid line).
To be specific, the RMSE of SVR is 0.013, while our DANK model achieves a promising approximation error with a value of 0.004.

\begin{figure*}[!htb]
	\centering
	\subfigure[recovering a step function]{\label{figstep}
		\includegraphics[width=0.32\textwidth]{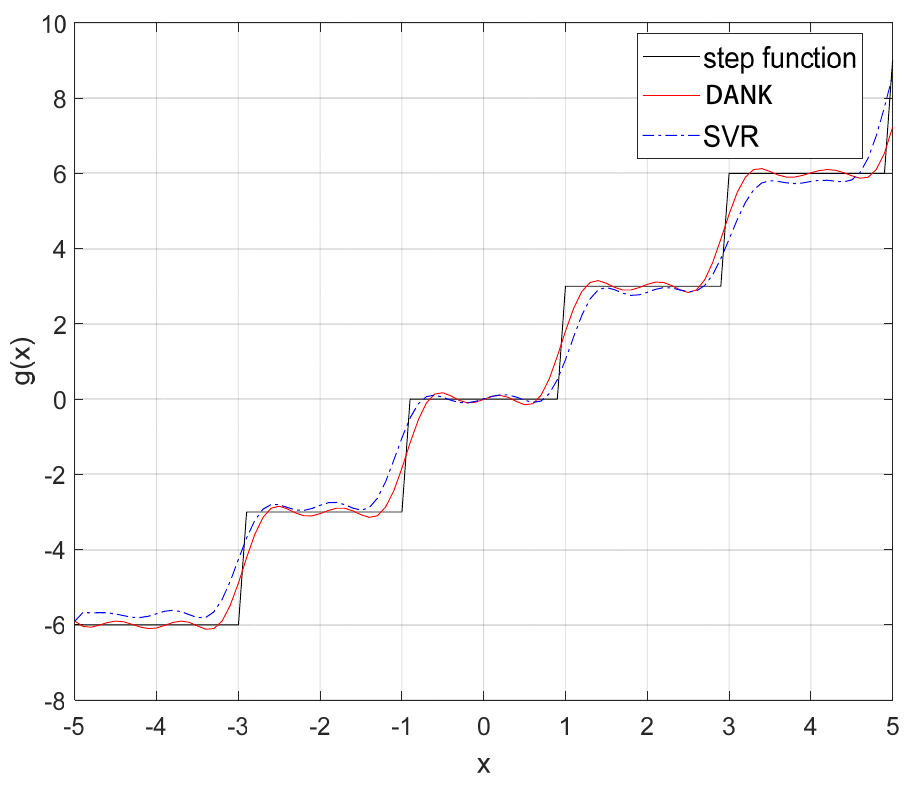}}
	\hspace{0.5cm}
	\subfigure[a 2-D function]{\label{fig2d}
		\includegraphics[width=0.32\textwidth]{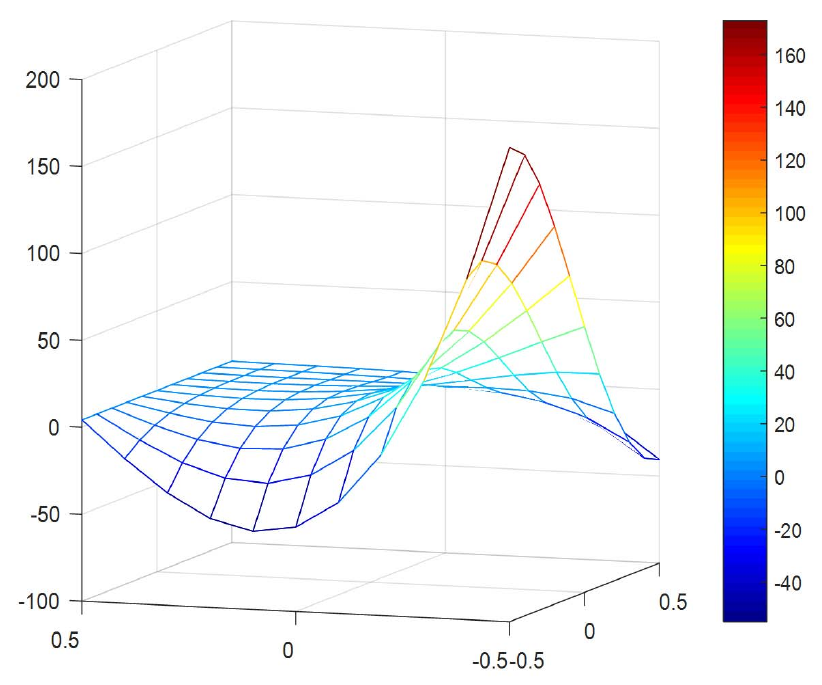}}
	\subfigure[recovered by SVR-CV]{\label{fig2dsvr}
		\includegraphics[width=0.32\textwidth]{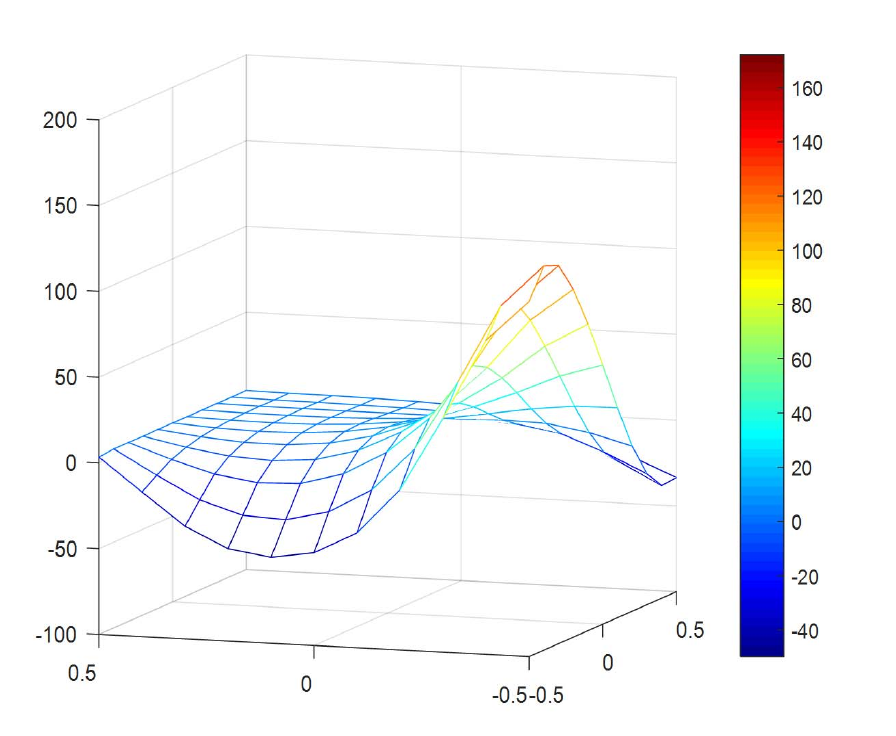}}
	\hspace{0.5cm}
	\subfigure[recovered by DANK]{\label{fig2dknpl}
		\includegraphics[width=0.32\textwidth]{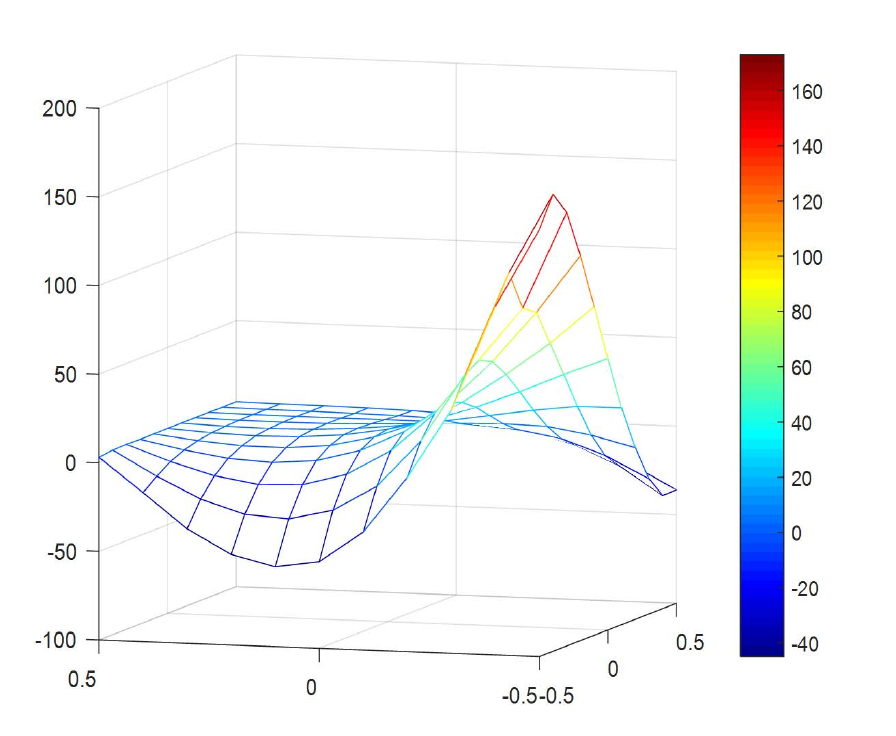}}
	\caption{Approximation for 1-D and 2-D functions by SVR-CV and our DANK model. }\label{figtoy}
\end{figure*}

Apart from the 1-D function, we use a 2-D test function to test SVR-CV and the proposed DANK model.
The 2-D test function \citep{cherkassky1996comparison} $g(u,v) \in [-0.5,0.5] \times [-0.5,0.5]$ is established as
\begin{equation*}
g(u,v) = 42.659\Big(0.1+(u - 0.5)\big( g_1(u, v) + 0.05\big) \Big)\,,
\end{equation*}
where $g_1(u, v)$ is defined by
\begin{equation*}
g_1(u, v) \!=\! (u - 0.5)^4 \!- 10(u \!- 0.5)^2 (v \!- 0.5)^2 \!+ 5 (v \!- 0.5)^4\,.
\end{equation*}
We uniformly sample 400 data points by $g(u,v)$ as shown in Figure~\ref{fig2d}, and then use SVR-CV and DANK to learn a regression function from the sampled data.
The regression results by SVR-CV and our DANK model are shown in Figure~\ref{fig2dsvr} and Figure~\ref{fig2dknpl}, respectively.
Intuitively, when we focus on the upwarp of the original function, our method is more similar to the test function than SVR-CV.
In terms of RMSE, the error of our method for regressing the test function is 0.007, which is lower than that of SVR-CV with 0.042. 

\begin{table*}
	\centering
	\scriptsize
	\caption{Comparison results of various methods on UCI regression data sets in terms of RMSE (mean$\pm$std. deviation). The best performance is highlighted in bold. The RMSE on the training data is presented by italic, and does not participate in ranking.  Notation ``$\bullet$" indicates that the method is significantly better than other methods via paired t-test at the 5\% significance level.}
	\begin{tabular}{cccccccccc}
		\toprule[1.5pt]
		\multirow{2}{*}{Data set} &\multirow{2}{*}{($d$, $n$)} &\multicolumn{1}{c}{BMKL } &\multicolumn{1}{c}{NW} &\multicolumn{2}{c}{SVR-CV} &\multicolumn{2}{c}{DANK}\cr
		\cmidrule(lr){3-4} \cmidrule(lr){5-6} \cmidrule(lr){7-8}
		{}&{}&Test&Test
		&Training&Test&Training&Test\cr
		\midrule[1pt]
		bodyfat &(14, 252)		&	0.101$\pm$0.065	&	0.097$\pm$0.010	&	\emph{0.007$\pm$0.000}	&	0.101$\pm$0.002	&	\emph{0.008$\pm$0.000}	&	{\bf 0.091}$\pm$0.013	\\
		\hline
		pyrim &(27, 74)	&	0.512$\pm$0.115	&	0.604$\pm$0.115	&	\emph{0.013$\pm$0.010}	&	0.678$\pm$0.255	&	\emph{0.007$\pm$0.004}	&	{\bf 0.457}$\pm$0.143$\bullet$	\\
		\hline
		space &(6, 3107)	&	0.248$\pm$0.107	 &	0.308$\pm$0.004	&	\emph{0.226$\pm$0.144}	&	0.261$\pm$0.112 &	\emph{	0.106$\pm$0.067}	&	{\bf 0.202}$\pm$0.114$\bullet$	\\
		\hline
		triazines &(60, 186)	&	{\bf 0.725}$\pm$0.124	&	0.885$\pm$0.039	&	\emph{0.113$\pm$0.084}	&	0.815$\pm$0.206	&	\emph{0.009$\pm$0.012}	&	{0.734$\pm$0.128}	\\
		\hline
		cpusmall	 &(12, 8912) 	&	0.133$\pm$0.004	& {\bf 0.037}$\pm$0.009	&	\emph{0.037$\pm$0.004}	&	0.104$\pm$0.002	&	\emph{0.001$\pm$0.000}	&	{0.114}$\pm$0.003	\\
		\hline
		housing &(13, 506)		&	0.286$\pm$0.027	&	{\bf 0.177}$\pm$0.015$\bullet$	&	\emph{0.085$\pm$0.049}	&	0.267$\pm$0.023	&	\emph{0.069$\pm$0.013}	&	0.218$\pm$0.013	\\
		\hline
		mg &(6, 1385)		&	0.297$\pm$0.013	&	{\bf 0.294}$\pm$0.010	&	\emph{0.163$\pm$0.076}	&	{0.407}$\pm$0.118	&	\emph{0.134$\pm$0.055}	&{0.303}$\pm$0.058\\
		\hline
		mpg &(7, 392)		&	0.187$\pm$0.014	&	0.182$\pm$0.012	&	\emph{0.095$\pm$0.087}	&	0.193$\pm$0.006	&	\emph{0.061$\pm$0.010 }&	{\bf 0.173}$\pm$0.008$\bullet$	\\
		\bottomrule[1.5pt]
	\end{tabular}
	\label{UCIreg}
\end{table*}

\subsubsection{Regression Results on UCI data sets}
We compare the proposed DANK model with other representative regression algorithms on eight data sets from the UCI database.
\cite{Gonen2012Bayesian} extend BMKL to regression tasks, and thus we include it for comparisons.
Apart from SVR-CV and BMKL, the Nadaraya-Watson (NW) estimator with metric learning \citep{noh2017generative} is also taken into comparison.
The remaining experimental settings follow with the classification tasks on the UCI database illustrated in Section~\ref{sec:classUCI}.

Table~\ref{UCIreg} lists a brief statistics of these eight data sets, and reports the average prediction accuracy and standard deviation of every compared algorithm.
Moreover, we present the regression performance of our DANK model and SVR-CV on the training data to show their
respective model flexibilities.
We find that, the proposed DANK model exhibits more encouraging performance than SVR-CV in terms of the RMSE on the training and test data.
From the results on four data sets including \emph{bodyfat}, \emph{pyrim}, \emph{space}, and \emph{mpg}, we observe that our method statistically achieves the best prediction performance.
On the remaining data sets, our DANK model achieves comparable performance with BMKL and NW. 

\section{Conclusion}
\label{sec:conclusion}
In this work, we propose an effective data-adaptive strategy to enhance the model flexibility for nonparametric kernel learning based algorithms.
In our DANK model, a multiplicative scale factor for each entry of the Gram matrix can be learned from the data, leading to an improved data-adaptive kernel.
As a result of such data-driven scheme, the model flexibility of our DANK model embedded in SVM (for classification) and SVR (for regression) is demonstrated by the experiments on synthetic and real data sets.
Besides, by introducing the low-rank constraint/regularizer and the bounded constraint/regularizer into our model, the learned kernel (matrix) exhibits a fast eigenvalue decay, which would be helpful to obtain good generalization properties. Accordingly, the used constraints/regularizers are able to provide a controllable trade-off between model flexibility and complexity.
Furthermore, our DANK model can be learned in an one-stage process with $\mathcal{O}(1/t^2)$ convergence rate due to the studied gradient-Lipschitz continuous property, where $t$ is the iteration number.
That means, we can simultaneously train the SVM or SVR along with optimizing the data-adaptive matrix by a projected gradient method with Nesterov's acceleration.
In addition, we develop a decomposed-based scalable approach to make our DANK model feasible to large data sets, of which the effectiveness has been verified by both experimental results and theoretical demonstration.

Extending our work to the case of deep kernel framework is an exciting avenue for future research, not limited to the ``shallow" kernel. When extending to a multi-layer framework, we could optimize the parameters in a layer-by-layer way during the training process.

\acks{
	This work was partly supported by National Key Research Development Project
	(No.2018AAA0100702), the National Natural Science Foundation of China (NSFC, No.61876107, U1803261, 61977046, 61973162), Committee of Science and Technology, Shanghai, China (No. 19510711200), the Fundamental Research Funds for the Central Universities (No: 30920032202), CCF-Tencent Open Fund (No: RAGR20200101), and the ``Young Elite Scientists Sponsorship Program" by CAST (No: 2018QNRC001).\\
	We thank Prof. Johan A.K. Suykens for some helpful suggestions on this paper.
	Jie Yang and Xiaolin Huang are corresponding authors.}

\appendices

\section{Proofs for Gradient-Lipschitz Continuity}
\subsection{Proofs of Lemma~\ref{gammabound}}
\label{proofgammabound}
\begin{proof}
	Since $\mathcal{J}_{\frac{\tau}{2}}$ is non-expansive \citep{Ma2011Fixed}, i.e., for any ${\bm \Omega}_1$ and ${\bm \Omega}_2$
	\begin{equation*}
	\| \mathcal{J}_{\frac{\tau}{2}}({\bm \Omega}_1) - \mathcal{J}_{\frac{\tau}{2}}({\bm \Omega}_2) \|_{\F} \leq \| {\bm \Omega}_1 - {\bm \Omega}_2 \|_{\F}\,,
	\end{equation*}
	with $\bm \Omega_1 := {\bm 1}{\bm 1}^{\top}+{\bm \Gamma}(\bm \alpha_1)$ and $\bm \Omega_2 := {\bm 1}{\bm 1}^{\top}+{\bm \Gamma}(\bm \alpha_2)$.
	Thereby, we have
	\begin{small}
		\begin{equation*}\label{boundgam}
		\begin{split}
		& \big\| \bm F(\bm \alpha_1) - \bm F(\bm \alpha_2)\big\|_{\F} \leq  \big\| {\bm \Gamma}(\bm \alpha_1) - {\bm \Gamma}(\bm \alpha_2) \big\|_{{\F}} \\
		&\!=\! \frac{1}{4\eta} \Big\|\! \diag\!\big({\bm \alpha_1}^{\!\!\!\top}\!\bm{Y}\big) \bm{K} \! \diag\!\big({\bm \alpha_1}^{\!\!\!\top}\!\bm{Y}\big) \!-\!\diag\!\big({\bm \alpha_2}^{\!\!\!\top}\!\bm{Y}\big) \bm{K} \! \diag\!\big({\bm \alpha_2}^{\!\!\!\top}\!\bm{Y}\big) \Big\|_{\F} \\
		&\!=\! \frac{1}{4\eta} \Big\| \diag\big({\bm \alpha_1}^{\!\top} \bm{Y} + {\bm \alpha_2}^{\!\top} \bm{Y}\big) \bm{K}\diag\big({\bm \alpha_1}^{\!\top}\bm{Y} - {\bm \alpha_2}^{\!\top} \bm{Y}\big)  \Big\|_{\F} \\
		& \!\leq\! \frac{1}{4\eta} \| \bm{K}\|_{\F} \Big\| \diag\!\big({\bm \alpha_1}^{\!\!\top} \!\bm{Y} \!+\! {\bm \alpha_2}^{\!\!\top}\! \bm{Y}\big) \Big\|_{\F} \Big\| \diag\big({\bm \alpha_1}^{\!\top} \bm{Y} \!-\! {\bm \alpha_2}^{\!\top} \bm{Y}\big) \Big\|_{\F} \\
		& \!\leq\! \frac{\|\bm{K}\|_\text{F}}{4\eta} \big\| {\bm \alpha_1} + {\bm \alpha_2} \big\|_2 \big\| {\bm \alpha_1} - {\bm \alpha_2} \big\|_2\,,
		\end{split}
		\end{equation*}
	\end{small}
	which yields the desired result.
\end{proof}

\subsection{Proofs of Theorem~\ref{theor}}
\label{prooftheor}
For notational simplicity, denote the shortcut $\bm F_1$ for $\bm F(\bm \alpha_1)$ and $\bm F_2$ for $\bm F(\bm \alpha_2)$. We will use them in the subsequent proof.
\begin{proof}
	The gradient of $h ({\bm \alpha})$ in Eq.~\eqref{falpha} is computed as
	\begin{equation}\label{gradf}
	\nabla h ({\bm \alpha}) = \bm{1} - \bm{Y}\big(\bm{F}({\bm \alpha})\odot\bm{K}\big)\bm{Y} {\bm \alpha}\,.
	\end{equation}
	It is obvious that $\bm{F}({\bm \alpha})$ is unique over the compact set $\mathcal{F}$.
	Hence, according to the differentiable property of the optimal value function \citep{penot2004differentiability}, for any $\bm \alpha_1$, $\bm \alpha_2$ $\in \mathcal{A}$, from the representation of $\nabla h(\bm \alpha)$ in Eq.~\eqref{gradf}, the function $h ({\bm \alpha})$ is proven to be gradient-Lipschitz continuous
	\begin{equation}\label{proof2}
	\begin{split}
	\| \nabla h(\bm \alpha_1) \!-\! \nabla h(\bm \alpha_2) \|_2 &=\big\| \bm{Y}\big(\bm{F}_1\odot\bm{K}\big)\bm{Y} {\bm \alpha_1} - \bm{Y}\big(\bm{F}_2\odot\bm{K}\big)\bm{Y} {\bm \alpha_2} \big\|_2 \\
	& = \big\| \bm{Y}\big(\bm{F}_1\! -\! \bm{F}_2 \big)\odot\bm{K}\bm{Y} {\bm \alpha_1} \!- \!\bm{Y}\big(\bm{F}_2\odot\bm{K}\big)\bm{Y} \big( {\bm \alpha_2}\!-\! {\bm \alpha_1} \big) \big\|_2 \\
	& \!\leq\! \big\| \bm{Y}\big(\bm{F}_1 - \bm{F}_2 \big)\odot\bm{K}\bm{Y} {\bm \alpha_1} \big\|_2 +  \big\| \bm{Y}\big(\bm{F}_2\odot\bm{K}\big)\bm{Y} \big( {\bm \alpha_2}- {\bm \alpha_1} \big) \big\|_2 \\
	&\! \leq\! \big\| \big(\bm{F}_1 - \bm{F}_2 \big)\odot\bm{K} \big\|_{2} \|  {\bm \alpha_1}\|_2 + \big\| \bm{F}_2\odot\bm{K} \big\|_{2} \| {\bm \alpha_1}- {\bm \alpha_2} \|_2~~~\\
	& \!\leq\! \kappa \big\| \bm{F}_1 - \bm{F}_2 \big\|_{\F}  \|  {\bm \alpha_1}\|_2 + \kappa \| \bm{F}_2 \|_{2}  \| {\bm \alpha_1}- {\bm \alpha_2} \|_2 \\
	&\! \leq \frac{ \kappa \|\bm{K}\|_{\F}}{4\eta} \|  {\bm \alpha_1} \|_2 \big(  \| {\bm \alpha_1} \|_2 + \| {\bm \alpha_2} \|_2 \!\big) \| {\bm \alpha_1}- {\bm \alpha_2} \|_2 +\kappa \lambda_{\max}(\bm F_2) \| {\bm \alpha_1}- {\bm \alpha_2} \|_2 \\
	& \!\leq\! \kappa \Big( n + \frac{ 3nC^2 \|\bm{K}\|_{\F}}{4\eta} \Big) \| {\bm \alpha_1}- {\bm \alpha_2} \|_2\,,
	\end{split}
	\end{equation}
	where the third inequality holds by $\| \bm{A} \odot \bm{B}\|_{2} \leq \| \bm A \|_2 \max_{ij} |B_{ij}| $ when $\bm B$ is PSD \citep{johnson1990matrix}.
	The fourth equality admits due to Lemma \ref{gammabound} and $\bm F_2 \in \mathcal{F}$.
	The last equality is achieved by $0 \leq {\bm \alpha} \leq C,~\|{\bm \alpha}\|^2_2\leq nC^2$ and  Lemma~\ref{fbound}.
	Hence, we conclude the proof.
\end{proof}

\subsection{Proofs of Theorem~\ref{theorsvr}}
\label{prooftheorsvr}
\begin{proof}
	For any $\hat{\bm \alpha}_1, \check{\bm \alpha}_1, \hat{\bm \alpha}_2, \check{\bm \alpha}_2 \in \mathcal{A}$, from the representation of $\nabla_{\hat{\bm \alpha}} h(\hat{\bm \alpha}, \check{\bm \alpha})$ in Proposition \ref{theorgradsvr}, we have
	\begin{equation}\label{proof2svr}
	\begin{split}
	&\Big\| \frac{\partial {{{h}(\hat{\bm \alpha}, \check{\bm \alpha}_1)}}}{\partial \hat{\bm \alpha}}\!\big|_{\hat{\bm \alpha} = \hat{\bm \alpha}_1} \!\!- \frac{\partial {{{h}(\hat{\bm \alpha}, \check{\bm \alpha}_2)}}}{\partial \hat{\bm \alpha}}\!\big|_{\hat{\bm \alpha} \!=\! \hat{\bm \alpha}_2} \Big \|_2 \!=\! \big\| \big({\bm{F}}_1 \!-\! {\bm{F}}_2 \big)\odot\bm{K} \big({\hat{\bm \alpha}_1 \!-\! \check{\bm \alpha}_1}\big) \!-\! \big({\bm{F}}_2\odot\bm{K}\big) \big( \hat{\bm \alpha}_2 \!-\! \check{\bm \alpha}_2 \!-\! \hat{\bm \alpha}_1 \!+\! \check{\bm \alpha}_1 \big) \big\|_2 \\
	&~~~~ \leq \kappa \big\| {\bm{F}}_1 - {\bm{F}}_2 \big\|_{\F} \| {\hat{\bm \alpha}_1- \check{\bm \alpha}_1}\|_2 + \kappa \| \bm{F}_2 \|_{2} \| \hat{\bm \alpha}_2- \check{\bm \alpha}_2 - \hat{\bm \alpha}_1+ \check{\bm \alpha}_1 \|_2 \\
	&~~~~ \leq \kappa \bigg( \frac{ \|\bm{K}\|_{\F}}{4\eta} \|  {\hat{\bm \alpha}_1- \check{\bm \alpha}_1} \|_2 \| \hat{\bm \alpha}_2- \check{\bm \alpha}_2 + \hat{\bm \alpha}_1- \check{\bm \alpha}_1 \|_2 +\lambda_{\max}(\bm F_2) \bigg) \| \hat{\bm \alpha}_2- \check{\bm \alpha}_2 - \hat{\bm \alpha}_1+ \check{\bm \alpha}_1  \|_2 \\
	&~~~~ \leq \kappa \bigg( \frac{ 8nC^2\|\bm{K}\|_{\F}}{4\eta} +n+\frac{nC^2}{4\eta}\lambda_{\max}(\bm K) \bigg) \bigg( \| \hat{\bm \alpha}_2 - \hat{\bm \alpha}_1 \|_2 +\| \check{\bm \alpha}_2 -\check{\bm \alpha}_1  \|_2 \bigg) \\
	&~~~~ \leq \kappa \Big( n + \frac{ 9nC^2 \|\bm{K}\|_{\F}}{4\eta} \Big)   \bigg( \| \hat{\bm \alpha}_2 - \hat{\bm \alpha}_1 \|_2 +\| \check{\bm \alpha}_2 - \check{\bm \alpha}_1  \|_2 \bigg)\,.
	\end{split}
	\end{equation}
	Similarly, $\Big\| \frac{\partial {{{h}(\hat{\bm \alpha}_1, \check{\bm \alpha})}}}{\partial \check{\bm \alpha}}\!\big|_{\check{\bm \alpha} = \check{\bm \alpha}_1} \!\!- \frac{\partial {{{h}(\hat{\bm \alpha}_2, \check{\bm \alpha})}}}{\partial \check{\bm \alpha}}\!\big|_{\check{\bm \alpha} = \check{\bm \alpha}_2} \Big \|_2 $ can also be bounded.
	Combining these two inequalities, we have
	\begin{equation}\label{gradsvrsub}
	\begin{split}
	\|\nabla_{ \!\tilde{\bm \alpha}_1} \! {h}(\hat{\bm \alpha}_1, \! \check{\bm \alpha}_1 \!) \! - \! \nabla_{ \!\tilde{\bm \alpha}_2}  \! {h}( \hat{\bm \alpha}_2, \! \check{\bm \alpha}_2 ) \|_2
	& \leq \Big\| \frac{\partial {{{h}(\hat{\bm \alpha}, \check{\bm \alpha}_1)}}}{\partial \hat{\bm \alpha}}\!\big|_{\hat{\bm \alpha} = \hat{\bm \alpha}_1} \!\!- \frac{\partial {{{h}(\hat{\bm \alpha}, \check{\bm \alpha}_2)}}}{\partial \hat{\bm \alpha}}\!\big|_{\hat{\bm \alpha} = \hat{\bm \alpha}_2} \Big \|_2 \\
	& + \Big\| \frac{\partial {{{h}(\hat{\bm \alpha}_1, \check{\bm \alpha})}}}{\partial \check{\bm \alpha}}\!\big|_{\check{\bm \alpha} = \check{\bm \alpha}_1} \!\!- \frac{\partial {{{h}(\hat{\bm \alpha}_2, \check{\bm \alpha})}}}{\partial \check{\bm \alpha}}\!\big|_{\check{\bm \alpha} = \check{\bm \alpha}_2} \Big \|_2 \\
	& \leq 2L  \Big(  \! \| \hat{\bm \alpha}_2  \!- \! \hat{\bm \alpha}_1 \|_2  \!+ \! \| \check{\bm \alpha}_2  \!- \! \check{\bm \alpha}_1  \|_2  \! \Big)\,,
	\end{split}
	\end{equation}
	which completes the proof.
\end{proof}

\section{Proofs in Large Scale Situations}
\subsection{Proofs of Lemma~\ref{lemmalink}}
\label{prooflemmalink}
\begin{proof}
	The quadratic term with respect to $\bm \alpha$ in Eq.~\eqref{mainls} can be decomposed into
	\begin{equation*}
	\bm \alpha^{\!\!\top}\!\bm Y\! \big(\bm{F}\odot \bar{\bm{K}} \big)\bm{Y}\! {\bm \alpha} \!=\! \sum_{c=1}^{v}  {\bm \alpha^{(c)}}^{\!\!\top}\bm{Y}^{(c,c)}\big(\bm{F}^{(c,c)}\odot\bm{K}^{(c,c)}\big) \bm{Y}^{(c,c)}\! {\bm \alpha}^{(c)}\,,
	\end{equation*}
	and $\| \bm{F} - \bm{1} \bm{1}^{\!\top} \|_{{\F}}^2$ can be expressed as
	\begin{equation*}
	\| \bm{F} - \bm{1} \bm{1}^{\!\top} \|_{{\F}}^2 = \sum_{i,j=1}^{n} (F_{ij}-1)^2 = \sum_{c=1}^{v} \| \bm F^{(c,c)} - \bm{1} \bm{1}^{\top} \|_{{\F}}^2 + \mbox{Const}\,,
	\end{equation*}
	where the constant is the sum of non-block-diagonal elements of $\bar{\bm F}$, and it does not affect the solution of Eq.~\eqref{mainlslemma}.
	Specifically, the positive semi-definiteness on $\bar{\bm F}^{(c,c)}$ with $c=1,2,\dots,v$ still guarantees that the whole matrix $\bar{\bm F}$ is PSD.
	Besides, the constraint in Eq.~\eqref{mainlslemma} is also decomposable, so the subproblems are separable and independent.
	As a result, the concatenation of their optimal solutions yields the optimal solution of Eq.~\eqref{mainlslemma}.
\end{proof}

\subsection{Proofs of Theorem~\ref{theorappf}}
\label{prooftheorappf}
\begin{proof}
	Denote the objective function in Eq.~\eqref{mainlslemma} with the Gram matrix $\bar{\bm K}$ as $\bar{H}(\bm \alpha, \bm F)$, its optimal solution $(\bar{\bm \alpha},\bar{\bm{F}})$ is indeed a saddle point due to the max-min problem in Eq.~\eqref{mainlslemma}.
	It is easy to check $\bar{H}({\bm \alpha},\bar{\bm{F}}) \leq \bar{H}(\bar{{\bm \alpha}},\bar{\bm{F}}) \leq \bar{H}(\bar{{\bm \alpha}},\bm{F})$ for any feasible $\bm \alpha$ and $\bm F$.
	
	Defining $H(\bm \alpha^*, \bm F^*)$ and $\bar{H}(\bm \alpha^*, \bm F^*)$, we can easily obtain
	\begin{equation}\label{hbarstar}
	\bar{H}(\bm \alpha^*, \bm F^*) \!-\! H(\bm \alpha^*, \bm F^*) \!= \! \frac{1}{2} \sum_{i,j:\pi(\bm x_i) \neq \pi(\bm x_j) }^{n} \!  \alpha_i^*
	\alpha_j^* y_i y_j  F^*_{ij}  K_{ij}\,,
	\end{equation}
	with $ F^*_{ij} := \bm F^*(\bm x_i, \bm x_j)$.
	Similarly, we have
	\begin{equation*}\label{hbarbara}
	\bar{H}(\bar{\bm \alpha}, \bm F^*)\! - \!H(\bar{\bm \alpha}, \bm F^*) \!=\! \frac{1}{2} \sum_{i,j:\pi(\bm x_i) \neq \pi(\bm x_j) }^{n} \bar{\alpha}_i
	\bar{\alpha}_j y_i y_j F^*_{ij} K_{ij}\,.
	\end{equation*}
	Therefore, combining above equations, the upper bound of $H(\bm \alpha^*, \bm F^*) - {H}(\bar{\bm \alpha}, \bar{\bm F})$ can be derived by
	\begin{equation*}\label{process1}
	\begin{split}
	H(\bm \alpha^*, \bm F^*)& \leq \bar{H}(\bar{\bm \alpha}, \bm F^*) - \frac{1}{2} \sum_{i,j:\pi(\bm x_i) \neq \pi(\bm x_j) }^{n} \alpha_i^*
	\alpha_j^* y_i y_j F^*_{ij} K_{ij} \\
	&\! = \!H(\bar{\bm \alpha}, \bm F^*) \!-\! \frac{1}{2} \sum_{i,j:\pi(\bm x_i) \neq \pi(\bm x_j) }^{n} \Big( \alpha_i^*
	\alpha_j^* - \bar{\alpha}_i
	\bar{\alpha}_j \Big) y_i y_j F^*_{ij} K_{ij}\\
	& \!\leq\! {H}(\bar{\bm \alpha}, \bar{\bm F}) -  \frac{1}{2} \sum_{i,j:\pi(\bm x_i) \neq \pi(\bm x_j) }^{n} \Big(\alpha_i^*
	\alpha_j^* - \bar{\alpha}_i
	\bar{\alpha}_j \Big) y_i y_j F^*_{ij} K_{ij}\\
	&\! \leq \! {H}(\bar{\bm \alpha}, \bar{\bm F}) + \frac{1}{2}BC^2 Q(\pi)\,,
	\end{split}
	\end{equation*}
	where the first inequality holds by $\bar{H}(\bm \alpha^*, \bm F^*) \leq \bar{H}(\bar{\bm \alpha}, \bm F^*)$ and Eq.~\eqref{hbarstar}.
	The second inequality admits by $H(\bar{\bm \alpha}, \bm F^*) \leq {H}(\bar{\bm \alpha}, \bar{\bm F})$.
	Similarly, $H(\bm \alpha^*, \bm F^*) - {H}(\bar{\bm \alpha}, \bar{\bm F})$ is lower bounded by
	\begin{equation*}
	\begin{split}
	H(\bm \alpha^*, \bm F^*) &= \bar{H}({\bm \alpha}^*, \bm F^*) \!-\! \frac{1}{2} \sum_{i,j:\pi(\bm x_i) \neq \pi(\bm x_j) }^{n} \alpha_i^*
	\alpha_j^* y_i y_j F^*_{ij} K_{ij} \\
	& \geq \bar{H}({\bm \alpha}^*, \bar{\bm F}) - \frac{1}{2} \sum_{i,j:\pi(\bm x_i) \neq \pi(\bm x_j) }^{n} \alpha_i^*
	\alpha_j^* y_i y_j F^*_{ij} K_{ij} \\
	& = {H}({\bm \alpha}^*, \bar{\bm F}) + \frac{1}{2} \sum_{i,j:\pi(\bm x_i) \neq \pi(\bm x_j) }^{n} \alpha_i^*
	\alpha_j^* y_i y_j (\bar{ F}_{ij} - F^*_{ij} ) K_{ij} \\
	& \geq {H}(\bar{\bm \alpha}, \bar{\bm F}) + \frac{1}{2} \sum_{i,j:\pi(\bm x_i) \neq \pi(\bm x_j) }^{n} \alpha_i^*
	\alpha_j^* y_i y_j (\bar{F}_{ij} - F^*_{ij} ) K_{ij} \\
	& \geq {H}(\bar{\bm \alpha}, \bar{\bm F}) - \frac{1}{2}BC^2 Q(\pi)\,,
	\end{split}
	\end{equation*}
	which concludes the proof that $\left| H(\bm \alpha^*, \bm F^*) - H(\bar{\bm \alpha},\!\bar{\bm{F}}) \right| \leq \frac{1}{2}BC^2Q(\pi)$.
\end{proof}

\subsection{Proofs of Theorem~\ref{theorsv}}
\label{prooftheorsv}
\begin{proof}
	We decompose  $\nabla_{\bm \alpha} {H}({\bm \alpha}^*, \bm F^*)$ into
	\begin{equation*}
	\begin{split}
	\Big(\! \nabla_{\bm \alpha} {H}({\bm \alpha}^*, \bm F^*) \!\Big)_i  &= \Big(\! \nabla_{\bm \alpha} \bar{H}(\bar{\bm \alpha}, \bar{\bm F}) \! \Big)_{\!i}
	+ \Big(\!\bm Y(\bm F^* \odot \bar{\bm K} )\bm Y {\Delta \bm \alpha} \!\Big)_i
	+ \Big(\bm Y(\Delta \bm F \odot \bar{\bm K} )\bm Y \bar{\bm \alpha} \Big)_i \\
	& \quad - \sum_{j:\pi(\bm x_i) \neq \pi(\bm x_j) }^{n} \alpha^*_j
	y_i y_j F^*_{ij} K_{ij} \\
	& \leq \Big( \nabla_{\bm \alpha} \bar{H}(\bar{\bm \alpha}, \bar{\bm F}) \Big)_i +B_2 \kappa C \!+\!(B_2\!-\!B_1) \kappa C + B_2 C \| \bar{\bm K}_i \|_1  \\
	& \leq \Big( \nabla_{\bm \alpha} \bar{H}(\bar{\bm \alpha}, \bar{\bm F}) \Big)_i \!+ B_2C \| \bar{\bm K}_i \|_1 + (B\!+\!B_2) \kappa C \! \\
	& \leq \Big( \nabla_{\bm \alpha} \bar{H}(\bar{\bm \alpha}, \bar{\bm F}) \Big)_i + (B + B_2)C (\| \bm K \|_1 + \kappa) \,,
	\end{split}
	\end{equation*}
	where $\bm K_i$ denotes the $i$-th column of the kernel matrix $\bm K$.
	We require $\big( \nabla_{\bm \alpha} {H}({\bm \alpha}^*, \bm F^*) \big)_i  \leq 0$ when $\bar{\bm \alpha}_i =0$.
	To this end, we need the right-hand of the above inequality is smaller than zero.
	As a result, we can conclude that ${\bm \alpha}_i =0$ from the optimality condition of problem~\eqref{mainls}.
\end{proof}


\newpage

\vskip 0.2in

\end{document}